\documentclass{article}

\usepackage{multirow}
\usepackage{nicefrac}
\usepackage{microtype}
\usepackage{graphicx}
\usepackage{subfigure}
\usepackage{booktabs} 
\usepackage{tablefootnote}
\usepackage{hyperref}

\usepackage[accepted]{icml2024}

\usepackage{amsmath}
\usepackage{amssymb}
\usepackage{mathtools}
\usepackage{amsthm}

\usepackage[capitalize,noabbrev]{cleveref}

\theoremstyle{plain}
\newtheorem{theorem}{Theorem}[section]

\newtheorem{lemma}[theorem]{Lemma}
\newtheorem{corollary}[theorem]{Corollary}
\theoremstyle{definition}

\theoremstyle{remark}

\newtheorem*{theorem*}{Theorem}
\newtheorem*{lemma*}{Lemma}
\newtheorem{corollary*}{Corollary}

\usepackage[textsize=tiny]{todonotes}

\def\cS{{\mathcal{S}}}
\def\s{{\mathcal{S}}}
\def\bS{{\mathbb{S}}}

\def\cD{{\mathcal{D}}}
\def\cA{{\mathcal{A}}}

\def\cL{{\mathcal{L}}}
\def\cF{{\mathcal{F}}}

\def\cZ{{\mathcal{Z}}}
\def\hG{\hat{G}}
\def\rP{{\mathrm{P}}}

\def\EE{{\mathbb{E}}}

\def\RR{{\mathbb{R}}}

\icmltitlerunning{Constrained Ensemble Exploration for Unsupervised Skill Discovery}

\begin{document}

\twocolumn[
\icmltitle{Constrained Ensemble Exploration for Unsupervised Skill Discovery}




\begin{icmlauthorlist}
\icmlauthor{Chenjia Bai}{shlab,sz}
\icmlauthor{Rushuai Yang}{hkust}
\icmlauthor{Qiaosheng Zhang}{shlab}
\icmlauthor{Kang Xu}{tencent}
\icmlauthor{Yi Chen}{hkust}
\icmlauthor{Ting Xiao}{ecust}
\icmlauthor{Xuelong Li}{shlab,teleai}
\end{icmlauthorlist}

\icmlaffiliation{shlab}{Shanghai Artificial Intelligence Laboratory}
\icmlaffiliation{hkust}{Hong Kong University of Science and Technology}
\icmlaffiliation{tencent}{Tencent}
\icmlaffiliation{ecust}{East China University of Science and Technology}
\icmlaffiliation{teleai}{The Institute of Artificial Intelligence (TeleAI), China Telecom}
\icmlaffiliation{sz}{Shenzhen Research Institute of Northwestern Polytechnical University}

\icmlcorrespondingauthor{Ting Xiao}{xiaoting@ecust.edu.cn}

\icmlkeywords{Reinforcement Learning, Skill Discovery, Entropy, Cluster}

\vskip 0.3in
]



\printAffiliationsAndNotice{}

\begin{abstract}
Unsupervised Reinforcement Learning (RL) provides a promising paradigm for learning useful behaviors via reward-free per-training. Existing methods for unsupervised RL mainly conduct empowerment-driven skill discovery or entropy-based exploration. However, empowerment often leads to static skills, and pure exploration only maximizes the state coverage rather than learning useful behaviors. In this paper, we propose a novel unsupervised RL framework via an ensemble of skills, where each skill performs partition exploration based on the state prototypes. Thus, each skill can explore the clustered area locally, and the ensemble skills maximize the overall state coverage. We adopt state-distribution constraints for the skill occupancy and the desired cluster for learning distinguishable skills. Theoretical analysis is provided for the state entropy and the resulting skill distributions. Based on extensive experiments on several challenging tasks, we find our method learns well-explored ensemble skills and achieves superior performance in various downstream tasks compared to previous methods.
\end{abstract}

\section{Introduction}

Reinforcement Learning (RL) \cite{RLBook-2018} has demonstrated strong abilities in decision-making for various applications, including game AI \cite{Go,efficientZero}, self-driving cars \cite{driving}, robotic manipulation \cite{TD-MPC,he2024diffusion,he2024large}, and locomotion tasks \cite{quad-science-2022,shi2023robust}. However, most successes rely on well-defined reward functions based on physical prior and domain knowledge \cite{haldar2022watch}, which can be notoriously difficult to design \cite{kwon2023reward}. In contrast to RL, other research fields like language and vision have greatly benefited from unsupervised learning (i.e., without annotations or labels), such as auto-regressive pre-training for Large Language Model (LLM) \cite{NLP-1,llama} and unsupervised representation learning for images \cite{simclr,diffZero} that benefit various language and vision tasks. Motivated by this, unsupervised RL aims to learn meaningful behaviors without extrinsic rewards, where the learned behaviors can be used to solve various downstream tasks via fast adaptation for generalizable RL. 

In unsupervised RL research, previous methods often conduct empowerment-driven skill discovery to learn distinguishable skills \cite{vic,diayn}. Specifically, the agent learns skill-conditional policies by maximizing an estimation of Mutual Information (MI) between skills and trajectories, which leads to discriminating skill-conditional policies with different behaviors. However, such an MI objective often generates static skills with poor state coverage \cite{disdain}. Recent works partially address this problem via Lipschitz constraints \cite{LSD-2022,USD-2023} and random-walk guidance \cite{guidance-2023}, while they still rely on the primary MI objective. Meanwhile, estimating the MI needs variational estimators based on sampling \cite{MI-estimator}, which is challenging in high-dimensional and stochastic environments \cite{becl} and also leads to sub-optimal performance \cite{URLB}. Other methods perform pure exploration via curiosity \cite{burda2018rnd} and state entropy \cite{aps,apt,proto} in environments, while they only focus on maximizing the state coverage rather than learning meaningful behaviors for downstream tasks. 

In this paper, we take an alternative perspective for unsupervised RL and propose a novel skill discovery framework, named \underline{C}onstrained \underline{E}nsemble exploration for \underline{S}kill \underline{D}iscovery (CeSD). We adopt an ensemble of value functions to learn different skills, where each value function uses independent intrinsic rewards that encourage the agent to explore a partition of the state space based on the assigned prototype, without considering the states of other prototypes.
The prototypes are learned by feature clustering of visited states and can act as representative anchors in the state visitation space. Based on the ensemble value function, we obtain the corresponding skills via policy gradient updates. Since the skills perform entropy estimation based on non-overlapping clusters, they can perform independent exploration to expand the boundary of the assigned cluster, leading to diverse behaviors. 
To overcome the potential overlap of the state coverage of updated skills, we adopt additional constraints to the state distribution between skills and the assigned clusters, which enforce skills to visit non-overlapping states to generate more distinguishable skills. Theoretically, we show the state entropy of each skill is monotonically increasing with the distribution constraints, and the ensemble skills maximize the global state coverage via partition exploration in clusters. We conduct extensive experiments on mazes and Unsupervised Reinforcement Learning Benchmark (URLB) \cite{URLB}, showing that CeSD learns well-explored and diverse skills. 

The contribution can be summarized as follows. (\romannumeral 1)
Unlike previous empowerment-based methods, CeSD takes an alternative perspective on skill discovery that bypasses MI estimation and also learns meaningful skills assisted by entropy-based exploration. (\romannumeral 2) We propose ensemble skills that explore the environment within individual clusters and apply additional constraints to learn distinguishable skills. (\romannumeral 3) We provide theoretical analysis for the state coverage of skills. (\romannumeral 4) We obtain state-of-the-art performance in various downstream tasks from challenging DeepMind Control Suite (DMC) tasks of URLB. The open-sourced code is available at \url{https://github.com/Baichenjia/CeSD}.

\section{Preliminaries}

We consider a Markov Decision Process (MDP) with an additional skill space, defined as $(\cS, \cA, \cZ, P, r, \gamma, \rho_0)$, where $\cS$ is the state space, $\cA$ is the action space, $\cZ$ is a skill space,  $P(s'|s,a)$ is the transition function, $\gamma$ is the discount factor, and $\rho_0:\cS \rightarrow[0,1]$ is the initial state distribution.
We use a discrete skill space $\cZ$ since learning infinite skills with diverse behaviors can be difficult \cite{recurrent-2022}. When interacting with the environment, the agent takes actions $a\sim \pi(\cdot|s,z)$ by following the skill-conditional policy $\pi(a|s,z)$ with a one-hot skill vector $z\in \RR^{n}$. We use $z_i$ to denote the vector with a $1$ in the $i$-th coordinate and $0$'s elsewhere. For example, $z_3=(0,0,1,0,0)$ in $\RR^5$. We use $\pi_i(a|s)$ and $\pi(a|s,z_i)$ interchangeable to denote the policy condition on skill $z_i$. Given clear contexts, we refer to the `skill-conditional policy' as `skill' for simplification. 

In the skill-learning stage, the policy is learned by maximizing discounted cumulative reward denoted as $\sum_t {\gamma^t}r_t$, where $r_t$ is generated by some intrinsic reward function, such as empowerment or entropy-based methods. In the policy-adaptation stage, we choose a specific skill vector $z^\star$ and fine-tune the policy $\pi(a|s,z^\star)$ with the extrinsic reward for downstream tasks. In unsupervised RL, we allow the agent to perform sufficient interactions in the skill-learning stage to learn meaningful skills, while only allowing a small number of interactions in the fine-tuning stage to perform policy adaptation. Overall, unsupervised RL aims to learn skills in the first stage for fast adaptation to various tasks in the second stage. 

We denote $I(\cdot;\cdot)$ by the mutual information between two random variables, and $H(\cdot)$ by either the Shannon entropy or differential entropy depending on the context. We use uppercase letters for random variables and lowercase letters for their realizations. The empowerment objective maximizes an MI-objective $I(S;Z)$ estimation and the entropy-driven objective maximizes $H(S)$. In both objectives, $s\sim d^\pi(s)$ is the normalized probability that a policy $\pi$ encounters state $s$, defined as $d^\pi(s) \triangleq (1-\gamma)\sum_{t=0}^\infty \gamma^t P(s_t=s|\pi)$.


\section{Method}

The proposed CeSD adopts ensemble $Q$-functions for skill discovery, where each skill performs partition exploration with prototypes. We adopt constraints on state distribution for regularizing skills. We give theoretical analyses to show the advantage of our algorithm on state coverage. 

\subsection{Ensemble Skill Discovery}

Previous methods learn skill-conditional policy $\pi(a|s,z)$ by maximizing the corresponding value function $Q_z(s,a)$. 
However, since different skills share the same network parameters, optimizing one skill can affect learning other skills. According to our observations, learning a single value function can have negative effects on learning diverse skills that have significantly different behaviors. 

To address this problem, we propose to use an ensemble of value functions in CeSD. Specifically, we adopt an ensemble of $Q$-networks for different skills, defined as $\{Q_1(s,a),\ldots,Q_n(s,a)\}$. The ensemble number is the same as the number of skills. Each $Q$-network is learned by minimizing the temporal-difference (TD) error as
\begin{equation}
\label{eq:tderror}
\min_{\phi_i} \EE_{\cD_i}\big[Q_{\phi_i}(s,a)-\big(r_i(s,a)+\gamma \max_{a'}Q_{\phi_i'}(s',a')\big)\big],
\end{equation}
where $\phi_i$ is the parameter of $i$-th network, $Q_{\phi_i'}$ is the corresponding target network, $\cD_i$ is a state buffer, and $r_i$ is the intrinsic reward and will be discussed later. Since different skills have independent parameters for the $Q$-function, the different $Q$-functions can emerge diverse behaviors through optimization. In training the $Q$-networks in Eq.~\eqref{eq:tderror}, we adopt efficient parallelization for ensemble networks to minimize the run-time increase with the number of skills. 

For learning the policy, we adopt a basic skill-conditional actor that maximizes the corresponding value function in the ensemble, and the objective function is
\begin{equation}
\label{eq:actor}
\max\nolimits_\psi \EE_{a\sim \pi_\psi(\cdot|s,z_i)} \big[Q_{\phi_i}(s,a)\big],\quad {i\in[n]} 
\end{equation}
where we denote $\psi$ as the policy parameters. Since the value ensemble has already learned the knowledge of different skills, 
we find that using a single network is sufficient to model the multi-skill policy. 

Although previous works have also adopted ensemble $Q$-networks \cite{Sunrize,EDAC} for online and offline RL, they are significantly different from our method. Specifically, previous methods use the \emph{same} objective for ensemble $Q$-networks. Thus, the learned $Q$-values estimate the approximate posterior of $Q$-function in online and offline RL \cite{PBRL,hyperdqn}, essential for theoretically grounded uncertainty estimation for optimism \cite{bai-1,bai-2} or pessimism \cite{wang2024ensemble,wen2024contrastive,bai2024pessimistic,deng2023false}. In contrast, we adopt `ensemble' to represent a collection of Q-functions used for different skills. These skills are learned in the state space via partition exploration and used for downstream adaptation. The ensemble $Q$-networks in our method have \emph{different} objectives that encourage independent exploration for separate areas with intrinsic rewards, which makes the ensembles represent value functions of diverse skills that optimize the policy in different directions. 

\subsection{Partition Exploration with Prototype}

We learn state prototypes through self-supervised learning to divide the explored states into clusters. Then, each $Q$-function in the ensemble can perform independent exploration based on the entropy estimation of states in the corresponding cluster. Specifically, we learn discrete state prototypes through soft-assignment clustering, and the learned prototypes act as representative anchors in the state space. Based on the prototypes, each visited state can be assigned to a specific cluster, and each cluster corresponds to a specific value function in exploration. 

The training of prototypes is given as follows. For a specific state $s_t$, we use a neural network to map the state to a vector $u_t=f_\theta(s_t)\in \RR^{m}$. We also define $n$ continuous vectors $\{c_1,\ldots,c_n\}$ as prototypes, where $c_i\in \RR^{m}$. Then the probability of $s_t$ being assigned to the $i$-th prototype is
\begin{equation}
\label{eq:cluster}
p^{(t)}_i=\exp\big(\hat{u}_t^\top c_i/ \tau\big) / \sum\nolimits_{j=1}^n \exp\big(\hat{u}_t^\top c_j/\tau\big),
\end{equation}
where $\hat{u}_t$ is the normalized vector as $u_t/\|u_t\|_2$, and $\tau$ is the temperature factor. Similar to Eq.~\eqref{eq:cluster}, we use a fixed target network $f_{\theta^-}(\cdot)$ with the same parameters as $f_{\theta}(\cdot)$ to obtain a normalized target vector $u^-_t=f_{\theta^-}(s_t)$. Then, the target probability $q^{(t)}_i$ is obtained by running an online clustering Sinkhorn-Knopp algorithm \cite{cluster-1,cluster-2} on the normalized target vector $\hat{u}^-_t$. Then, we use the cross-entropy loss to update the prototypes as 
\begin{equation}
\cL_{\rm proto}=-\sum\nolimits_{t}\sum\nolimits_{i} q^{(t)}_i \log p^{(t)}_i.    
\end{equation}
In the unsupervised stage, the prototypes $\{c_i\}$ will update with more collected states. Nevertheless, we remark that such an update is gradual with gradient descent and will not cause drastic changes in probability $p_i^{(t)}$, which makes the cluster assignment stable for the collected states and benefits the calculation of intrinsic rewards in exploration. 

Based on the learned prototypes, each state can be assigned to a specific cluster by following $z^{(t)}\sim \mathbf{p}^{(t)}$, where $\mathbf{p}^{(t)}=[p^{(t)}_1,\ldots,p^{(t)}_n]$ represents a categorical distribution. In practice, we adopt a small temperature value in Eq.~\eqref{eq:cluster} to obtain a near-deterministic cluster assignment. We denote the set of collected states by $\bS$, and then partition $\bS$ into $n$ disjoint clusters as $\{\bS_1,\bS_2,\ldots,\bS_n\}$
according to the categorical distribution. For convenience, we slightly abuse $\bS_i$ to include the whole transition $\{(s,a,s')\}$ for each state. Based on the clusters, the skill policies can conduct partition exploration by maximizing the state entropy of the corresponding cluster. Specifically, we adopt a simple cluster-skill correspondence mechanism by assigning the cluster $\bS_i$ to the value function $Q_i$ with the same skill index. Since the state entropy of each cluster can be estimated separately, we can calculate the intrinsic rewards for each value function independently to encourage partition exploration without considering states from other clusters. For example, the value function $Q_i$ will use $\{s,a,s'\}\in \bS_i$ and the corresponding intrinsic rewards $r^{\rm cesd}_i$ calculated in $\bS_i$ for TD-learning, which encourages policy $\pi(a|s,z_i)$ to explore the state space based on $\bS_i$ without considering other clusters (i.e., $\bS_j,j\neq i$), thus leading to diverse behaviors for different skills.

\begin{figure*}[t]
\begin{center}
\centerline{
\includegraphics[width=2.0\columnwidth]{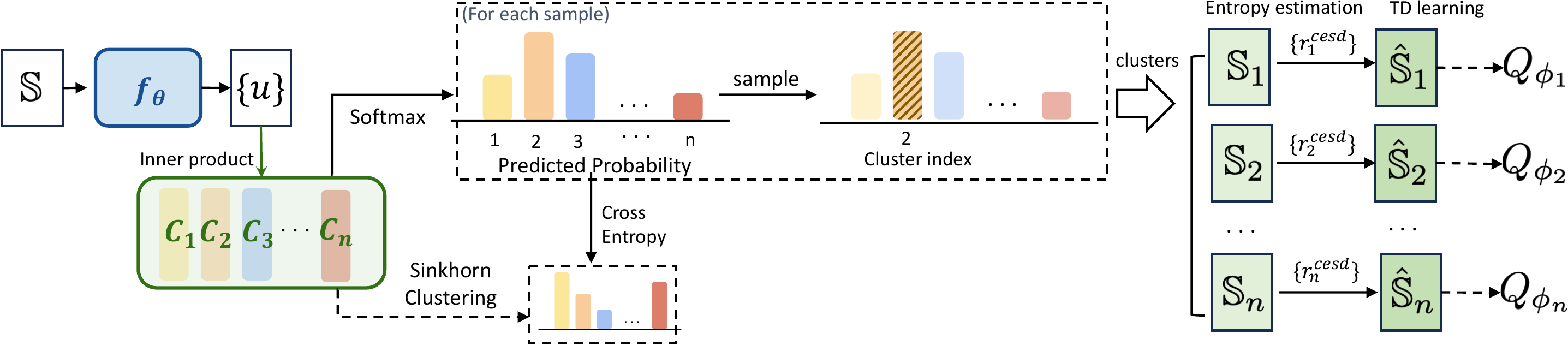}}
\vspace{-0.5em}
\caption{The partition exploration process. We adopt Sinkhorn-Knopp algorithm to learn prototypes and perform clustering for states. The intrinsic reward is calculated by entropy estimation within each cluster and then used for training a specific $Q$-network.}
\label{fig:cluster-process}
\vspace{-2em}
\end{center}
\end{figure*}

\paragraph{Particle Estimation} To calculate the entropy-based intrinsic reward $r^{\rm cesd}_i$, we adopt a popular particle-based entropy estimation algorithm in previous methods \cite{apt,cic}, and the entropy is estimated by a sum of the log distance between each particle and its $k$-th nearest neighbor. Following this method, the particle entropy estimation for $i$-th cluster $\bS_i$ is calculated as, 
\begin{equation}
H_{k}(\bS_i)\propto \sum\nolimits_{s_t\in \bS_i}\ln\big\|u_t-{\rm NN}_{k,f_\theta}(u_t)\big\|,
\end{equation}
where the distance is calculated in the feature space of states. Then the intrinsic reward for $(s_t,a_t,s_{t+1})\in\bS_i$ is set to
\begin{equation}
\label{eq:entropy-rew}
r^{\rm cesd}_i(s_t,a_t)=\big\|u_{t+1}-{\rm NN}_{k,f_\theta}(u_{t+1})\big\|.
\end{equation}
For each value function $Q_i$ in the ensemble, we perform clustering based on prototypes and obtain $\bS_i=\{(s_t,a_t,s_{t+1})\}$. We follow Eq.~\eqref{eq:entropy-rew} to calculate the intrinsic reward for each example in $\bS_i$, and obtain the reward-augmented cluster set as $\hat{\bS}_i=\{(s_t,a_t,s_{t+1},r^{\rm cesd}_i)\}$. Then, we minimize the TD-error of $Q_i$ by following Eq.~\eqref{eq:tderror} with experiences sampled from $\hat{\bS}_i$. We adopt the same training process for all clusters $i\in[n]$, which can be practically implemented via a masking technique to determine whether a transition should be used for training a specific $Q_i$ network. We illustrate the whole process of partition exploration in Figure~\ref{fig:cluster-process}.

\paragraph{Entropy Analysis} We give a simple analysis for entropy estimation based on clusters. The state entropy of partition exploration is calculated in each cluster $\bS_i$, while in global exploration is calculated in $\bS$. Given fixed state sets, we denote policies that obtain the maximum entropy in the cluster $\bS_i$ and the overall state set $\bS$ by $\pi_i^*$ and $\pi^*$, respectively. Then the following Theorem holds. 
\begin{theorem}
\label{thm:entropy}
Let each cluster have the same number of samples, for $i\in[n]$, the relationship between the maximum entropy of $\pi^*$ in the state set $\bS$ and $\pi_i^*$ in the cluster set $\bS_i$ is 
\begin{equation}
H\big(d^{\pi^*}(s)\big)=H\big(d^{\pi^*_i}(s)\big)+C(n),
\end{equation}
where $C(n)=\log n$ depends on the number of clusters $n$. 
\end{theorem}
The assumption holds since the Sinkhorn-Knopp algorithm constrains assigning each cluster to the same number of samples. We refer to Appendix~\ref{app:theory} for a proof. Theorem~\ref{thm:entropy} shows the optimal policies $\{\pi_i^*\}$ with uniform visitation in cluster sets $\{\bS_i\}$ also obtains the maximum entropy in the global state set $\bS$. Thus, performing partition exploration in clusters also maximizes the global state coverage. Meanwhile, we can obtain diverse skills through partition exploration rather than a single exploratory policy in global exploration \cite{apt,cic}. 


\subsection{Skill Distribution Constraint}

We delve into the learning process of partition exploration and propose a state-distribution constraint for generating distinguishable skills. In Figure~\ref{fig:constraint}, we show the information diagrams of the learning process of CeSD. For a clear illustration, we only show three skills in each sub-figure, while we may adopt more skills in practice for complex tasks. 

The randomly initialized skills often have small entropy (i.e., $H(d^{\pi_i}(s))$), which signifies each skill only visits states around the start point, as shown in Figure~\ref{fig:constraint}(a). 
Considering in a tabular case, we use $\{\bS^{\rm init}_i,\bS^{\rm init}_2,\bS^{\rm init}_3\}$ to represent collected state sets for three skills and assume $\bS^{\rm init}_i\cap \bS^{\rm init}_j=\emptyset$. Specifically, we assume each skill has an independent explored area initially since the corresponding value function in the ensemble has different initialized parameters.

\begin{figure*}[t]
\begin{center}
\centerline{
\includegraphics[width=2.0\columnwidth]{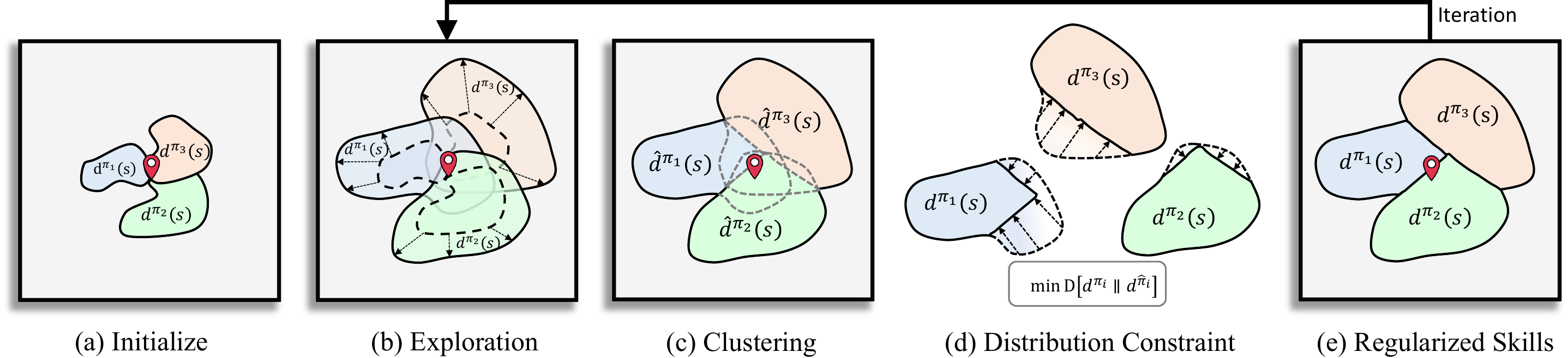}}
\vspace{-0.5em}
\caption{The learning process of CeSD. After initializing skills, we conduct entropy-based exploration for each skill and perform clustering to obtain non-overlapping clusters. Then the state distribution constraint is applied to enhance the diversity of skills. The regularized skills are used for partition exploration in the next round of iteration.}
\label{fig:constraint}
\vspace{-2em}
\end{center}
\end{figure*}

\paragraph{Problem Statement} In partition exploration, each skill uses the state entropy as intrinsic rewards defined in Eq.~\eqref{eq:entropy-rew}. Then, each skill $\pi(\cdot|s,z_i)$ will learn to (\romannumeral1) assign uniform occupancy probability for the visited states, which increases the entropy with a fixed state set. More importantly, (\romannumeral2) each skill will learn to explore the environment to collect \emph{new} states that do not occur in $\bS^{\rm init}_i$. As more states are added to $\bS^{\rm init}_i$, the corresponding state entropy $H(d^{\pi_i}(s))$ increases and the skill explores more unknown areas. As shown in Figure~\ref{fig:constraint}(b), the state coverage of each skill increased in partition exploration. We denote the new state sets after a round of \underline{p}artition \underline{e}xploration as $\{\bS_i^{\rm pe}\}$, and the overall state set is defined as $\bS^{\rm pe}=\cup_i \{\bS_i^{\rm pe}\}$. 

Nevertheless, since different skills perform independent exploration based on their previous state visitation, they may explore the same intersection area after policy update and collect the same states in the updated sets $\{\bS_i^{\rm pe}\}$, which makes $\bS_i^{\rm pe}\cap \bS_j^{\rm pe}\neq \emptyset$, where $i\neq j$. As shown in Figure~\ref{fig:constraint}(b), the visitation area of each skill overlaps with other skills, which does not hurt exploration but reduces the distinguishability of different skills. For example, in locomotion and manipulation tasks, different skills generate similar behavior if their state distributions overlap significantly.

\paragraph{State Distribution Constraint} 
To address this challenge, we propose an explicit distribution constraint for the updated state distribution to regularize skill learning. To achieve this, we first perform clustering to divide the overlapping area and assign different parts to different skills. We denote the different state set after \underline{clu}stering as $\{\bS^{\rm clu}_i\}$, then we have $\bS^{\rm clu}_i \cap \bS^{\rm clu}_j=\emptyset$ ($i\neq j$) since each state will be assigned to a unique cluster. As shown in Figure~\ref{fig:constraint}(c), different colors represent the non-overlapping state sets after clustering, and $\bS^{\rm pe} = \cup_i \{\bS^{\rm clu}_i\}$ holds since the overall state set does not change. In most cases, we have $\bS^{\rm clu}_i \subseteq \bS^{\rm pe}_i$ since 
the clustering algorithm will keep the cluster-index of existing states fixed and re-assign the newly added states to different clusters. Nevertheless, it does not always hold since the prototypes can be sub-optimal in training. 

Based on the above analyses and Figure~\ref{fig:constraint}(b), we denote the state distribution lying on the state set $\bS^{\rm pe}_i$ as $d^{\pi_i}$. However, a more desired state set is $\bS^{\rm clu}_i$ in Figure~\ref{fig:constraint}(c), which has non-overlapping states with other clusters and leads to more diverse skills. Thus, we define a desired policy $\hat{\pi}_i$ based on $\pi_i$, where
$d^{\hat{\pi}_i}(s)=0$ for $s\in\bS^{\rm pe}_i-\bS^{\rm clu}_i$ that represents the difference between two sets, and $d^{\hat{\pi}_i}(s)=d^{\pi_i}(s)/\sum_{s\in \bS^{\rm clu}_i} d^{\pi_i}(s)$ for other states by re-normalizing $d^{\pi_i}$ in the cluster states. Ideally, our constraint for regularizing the skill behavior is defined as
\begin{equation}
\cL_{\rm reg}(\pi_\theta(s,z_i))=\frac{1}{2}\sum\nolimits_{s\in\bS^{\rm pe}_i} \big|d^{\hat{\pi}_i}(s) - d^{\pi_i}(s) \big|.
\end{equation}
However, it can be computationally expensive to minimize the Total Variation (TV) distance $\cL_{\rm reg}$ via density estimation of states \cite{SMM,MADE}. Alternatively, we propose to approximately reduce such a gap by minimizing $\EE_{s\sim d^{\pi_i}}[D_{\rm TV}(\hat{\pi}_i(\cdot|s)\|\pi(\cdot|s))]$, which servers an upper bound of the density discrepancy (as shown in Lemma~\ref{lemma:distribution} below). The main difference between $\hat{\pi}_i(\cdot|s)$ and $\pi(\cdot|s)$ is that, the desired policy $\hat{\pi}_i$ has \emph{zero} visitation probability for states $s\in \bS^{\rm pe}_i-\bS^{\rm clu}_i$, while the policy $\pi(\cdot|s)$ has some probability to visit states in the intersection set of skills. Thus, we propose a heuristic intrinsic reward to prevent the current policy $\pi_i$ from visiting states in $\bS^{\rm pe}_i-\bS^{\rm clu}_i$, as 
\begin{equation}
\label{eq:reg_r}
r^{\rm reg}_i=1/\big(\big|\bS^{\rm pe}_i-\bS^{\rm clu}_i\big|+\lambda\big).
\end{equation}
This reward is maximized when $|\bS^{\rm pe}_i-\bS^{\rm clu}_i|=0$, which signifies $\pi_i$ only visits state in its assigned cluster set $\bS^{\rm clu}_i$ that has no overlap to other clusters. As shown in Figure~\ref{fig:constraint}(d), maximizing $r^{\rm reg}_i$ will force skill $\pi_i$ to reduce the visitation probability of states lied in clusters of other skills, which makes policy $\pi_i$ closer to $\hat{\pi}_i$ that only visits states in $\bS_i^{\rm clu}$. We illustrate the regularized skills in Figure~\ref{fig:constraint}(e). 

\paragraph{Qualitative Analysis}
In the next, we give a qualitative analysis of the state entropy of skills in the learning process. We start with a lemma to show that minimizing the policy divergence can also reduce the constraint term $L_{\rm reg}(\pi_i)$.

\begin{lemma}
\label{lemma:distribution}
The divergence between state distribution is bounded on the average divergence of policies $\hat{\pi}_i$ and $\pi_i$, as
\begin{equation}
D_{\rm TV}\big(d^{\hat{\pi}_i}\|d^{\pi_i}\big)
\leq \frac{\gamma}{1-\gamma}\EE_{s\sim d^{\pi_i}}\big[D_{\rm TV}(\hat{\pi}_i(\cdot|s)\|\pi_i(\cdot|s))\big],
\end{equation}
where $D_{\rm TV}(\cdot\|\cdot)$ is the total variation distance.
\end{lemma}
We refer to Appendix~\ref{app:theory} for a proof. According to Lemma~\ref{lemma:distribution}, by minimizing the policy divergence (i.e., $D_{\rm TV}(\hat{\pi}_i\|\pi_i)$) via the intrinsic reward, the difference between state distribution (i.e., $D_{\rm TV}(d^{\hat{\pi}_i}\|d^{\pi_i})$) can be bounded. Then we have the following theorem for the entropy of state distributions.
\begin{theorem}
\label{thm:distance}
Assuming the distance between state distribution is bounded by $D_{\rm TV}\big(d^{\hat{\pi}_i}\|d^{\pi_i}\big)\leq \delta$, the entropy difference between state distribution can be bounded by
\begin{equation}
\big|H(d^{\hat{\pi}_i})-H(d^{\pi_i})\big|\leq \delta \log\big(|\bS^{\rm pe}_i|-1\big)+h(\delta).
\end{equation}
where $h(x) := -x \log(x) - (1-x) \log (1-x)$ is the binary entropy function.
\end{theorem}
The proof follows the coupling technique \cite{coupling} and Fano's inequality \cite{Fano}, as given in Appendix~\ref{app:theory}.
\begin{corollary}
The state entropy of each $\pi_i$ is monotonically increasing with partition exploration and constraints.
\end{corollary}

\begin{proof}
Assuming that $D_{\rm TV}\big(d^{\hat{\pi}_i} || d^{\pi_i}\big)\leq \delta$, we have $|H(d^{\hat{\pi}_i})-H(d^{\pi_i})|\leq f(\delta)$, where $f(\delta)=\delta \log(|\mathbb{S}^{\rm pe}_i|-1)+h(\delta)$ is a constant determined by the state cluster and the state distribution bound. Then we have $H(d^{\pi_i})\geq H(d^{\hat{\pi}_i}) - f(\delta)$. Corollary 4.4 holds since we maximize the state entropy (i.e., we set the intrinsic reward to $r^{\rm cesd}_i=H(d^{\hat{\pi}_i})$ in policy learning. The state entropy of the skill policy (i.e., $H(d^{\pi_i})$) also increases since $f(\delta)$ is a positive constant given a fix $\delta$ and the state set.
\end{proof}

In each iteration of CeSD, since we maximize the state entropy in each cluster (i.e., $H(d^{\hat{\pi}_i})$) via particle estimation, the state entropy of current policy (i.e., $H(d^{\pi_i})$) is also forced to be increased according to Theorem~\ref{thm:distance}. As a result, the state entropy of each skill $\pi_i$ is monotonically increasing with partition exploration and distribution constraints. Further, according to Theorem~\ref{thm:entropy}, since the maximum entropy in each cluster (i.e., $H(d^{\hat{\pi}_i})$) has a constant gap with the maximum entropy in the overall state set (i.e., $H(d^{\hat{\pi}})$), our method monotonically increases the global state coverage in exploration. 

\section{Related Works}

\textbf{Unsupervised Pretraining of RL}~
Unsupervised Pretraining methods in RL aim at obtaining prior knowledge from unlabeled data or reward-free interactions to facilitate downstream task learning~\cite{xie2022pretraining}. The methods primarily fall into three categories: Unsupervised Skill Discovery~(USD)~\cite{diayn,dads,LSD-2022,ajay2020opal,USD-2023,cic}, Data Coverage Maximization~\cite{apt,aps,proto,yarats2022don,SMM}, and Representation Learning~\cite{mazoure2021cross,yang2021representation,yuan2022pre,ghosh2023reinforcement}. Our work falls into the first category and intersects the data coverage maximization methods. To obtain skills with discriminating behaviors, existing USD research mainly relies on the MI objectives~\cite{vic,diayn}. However, the recent study~\cite{disdain} has revealed the pessimistic exploration problem of the MI paradigm. Thus, several works try to enhance the state coverage via Euclidean distance constraint~\cite{LSD-2022}, recurrent training~\cite{recurrent-2022}, discriminator disagreement~\cite{disdain}, controllability-aware objective~\cite{USD-2023}, and guidance skill~\cite{guidance-2023}. However, they still rely on MI objectives for skill discrimination ~\cite{disdain,guidance-2023} and require variational estimators based on sampling~\cite{MI-estimator}. In contrast, CeSD simultaneously enhances state coverage and skill discrimination by performing partition exploration and clustering-based skill constraint, without requiring inefficient recurrent training and state distance functions. Concerning the clustering method, a similar technique has also been utilized in prior RL pretraining works~\cite{proto,mazoure2021cross}. Motivated by different purposes, we perform clustering to guarantee distinct exploration regions of skills, instead of estimating state visitation or learning generalizable representation.

\textbf{Policy Regularization in RL}~
Policy regularization has played diverse roles in RL algorithms, including constraining the policy to close to the behavior policy in offline RL~\cite{awac,CRR,td3bc}, enhancing the exploration ability for online exploration~\cite{sac,agac,ris}, and reducing the policy distance to demonstrations~\cite{gail,t-rex,dwbc,smodice}. Formally, policy regularization regularizes the current policy to some target policy with a specific probabilistic form, or the target policy can be estimated via limited interactions~\cite{dapg} or offline datasets~\cite{smodice}.
In contrast, the target skill policy in our work does not have a specific form and is characterized by clustered states $\bS_i^{\rm clu}$ sampled from the state distribution $d^{\hat{\pi}_i}(s)$. Meanwhile, the target policy changes in each iteration with more sampled states, which makes our setting different from prior ones and particularly challenging. Thus, we propose a simple but effective reward to encourage each skill to visit fewer states lying in clusters of other skills, which regularizes the policy learning.

\textbf{Value Ensemble}~
To capture the epistemic uncertainty \cite{bai2021dynamic,qiu2022Contrastive_ucb} induced by the limited training data in RL, the value ensemble has been proposed to approximately estimate the posterior distribution of the value function in online~\cite{td3,redq,Sunrize,bai-1} and offline RL \cite{PBRL,hyperdqn,EDAC}. Recent works \cite{xu2024cross,wen2023towards} also adopt an ensemble for cross-domain policy adaptation and reuse \cite{xu2023value}. Different from these works, we utilize value ensemble to estimate the expected returns of different skills. To confront the non-stationary objective and mutual interference between skills, we adopt an independent value function for each skill.

\section{Experiments}

In this section, we compare the performance of unsupervised RL methods in challenging URLB tasks \cite{URLB}. We also conduct experiments in a maze to illustrate the learned skills in a continuous 2D space. We finally conduct visualizations and ablation studies of our method. 

\begin{figure*}[t]
\begin{center}
\centerline{
\includegraphics[width=1.9\columnwidth]{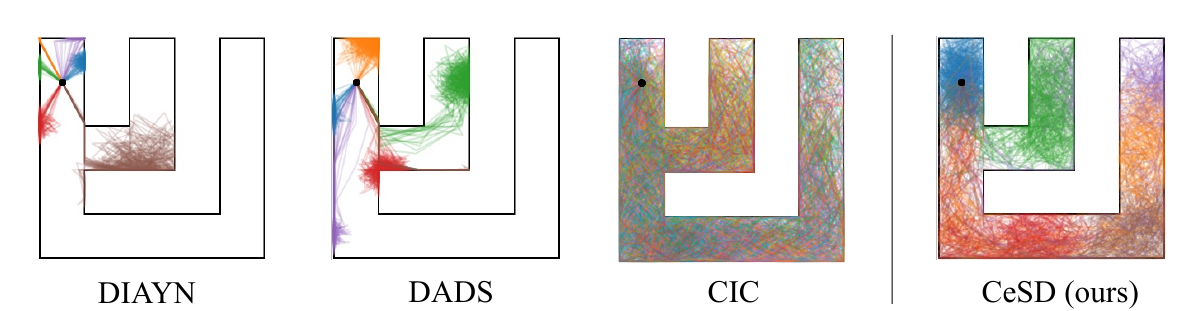}}
\vspace{-0.5em}
\caption{Visualization of skill discovery in Maze. Different colors represent the state trajectories with different skill vectors. We let the agent start moving from the black dot in the upper left corner and sample 20 trajectories for each skill for visualization.}
\label{fig:maze-result}
\vspace{-2em}
\end{center}
\end{figure*}

\subsection{Skill Learning in 2D maze}

We conduct experiments for skill discovery in a 2D maze environment from \citet{EDL-2020}. The observation of the agent is the current position $\cS\in\RR^2$, and the action $\cA\in\RR^2$ controls the velocity and direction. We consider several strong baselines for unsupervised training, including DIAYN \cite{diayn}, DADS \cite{dads}, and CIC \cite{cic}. In these methods, DIAYN and DADS perform skill discovery by maximizing the MI term of states and skills (i.e., $I(S;Z)$) via a reverse form  (i.e., $H(Z)-H(Z|S)$) and forward form (i.e., $H(S)-H(S|Z)$), respectively, to construct a variational estimation of the MI objective. CIC is a data-based method that maximizes the state entropy estimation (i.e., $H(S)$) for pure exploration. For a fair comparison, we use a 10-dimensional one-hot vector for skills in all methods. 

In Figure~\ref{fig:maze-result}, we visualize the learned skills by sampling trajectories from the skill-conditional policies of CeSD and other baselines. (\romannumeral1) Concerning the discriminability of skills, we find empowerment-based methods like DIAYN and DADS learn distinguishable skills, where each skill can generate trajectories that are different from those of other skills. In contrast, entropy-driven methods like CIC cannot generate discriminable skills due to the lack of mechanisms to distinguish different skills. (\romannumeral2) Concerning state coverage, both DIAYN and DADS have limited state coverage since they rely on the MI objective without encouraging exploration. The entropy-based CIC algorithm obtains the best state coverage since the entropy maximization encourages exploration of the environment and also leads to a uniform state visitation within the whole state space. (\romannumeral3) The proposed CeSD performs the best in both skill discriminability and state coverage. CeSD takes the theoretical advantage of monotonic entropy increasing (as Theorem~\ref{thm:distance}), and the ensemble skills obtain well global coverage. Meanwhile, CeSD learns distinguishable skills with approximately non-overlapping coverage via state distribution constraints.

\subsection{URLB Benchmark Results}

We evaluate CeSD in the URLB benchmark \cite{URLB}. \emph{Walker} domain contains biped locomotion tasks with $\cS\in\RR^{24}$ and $\cA\in\RR^{6}$; \emph{Quadruped} domain contains quadruped locomotion tasks with high-dimensional state and action space as   $\cS\in\RR^{78}$ and $\cA\in\RR^{16}$, which have much larger space in exploration and is more challenging; and \emph{Jaco Arm} domain contains a 6-DoF robotic arm and a three-finger gripper with $\cS\in\RR^{55}$ and $\cA\in\RR^{9}$. In experiments, each method performs unsupervised skill learning with intrinsic rewards and adapts the skills to downstream tasks with extrinsic rewards. There are 3 downstream tasks for each domain, including \emph{Stand, Walk, Run,} and \emph{Flip} tasks for \emph{Walker} domain, \emph{Stand, Walk, Run}, and \emph{Jump} tasks for \emph{Quadruped} domain, and move the objective to \emph{Bottom Left, Bottom Right, Top Left,} and \emph{Top Right} for \emph{Jaco Arm}.

\begin{figure*}[t]
\begin{center}
\centerline{
\includegraphics[width=2.1\columnwidth]{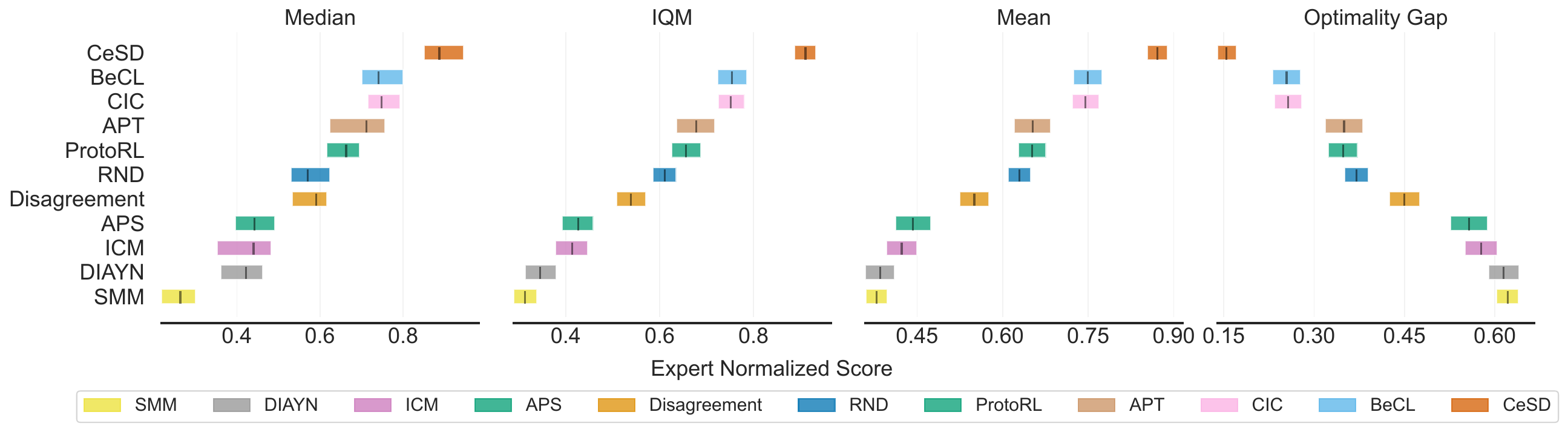}}
\vspace{-0.5em}
\caption{Comparison of performance in 12 downstream tasks of URLB benchmark. We report the aggregate statistics of 10 seeds by following \citet{agarwal2021IQM} after finetuning. CeSD achieves the new state-of-the-art results in the URLB benchmark.}
\label{fig:dmc-result}
\vspace{-1.5em}
\end{center}
\end{figure*}

\begin{figure*}[t]
\begin{center}
\centerline{
\includegraphics[width=2.0\columnwidth]{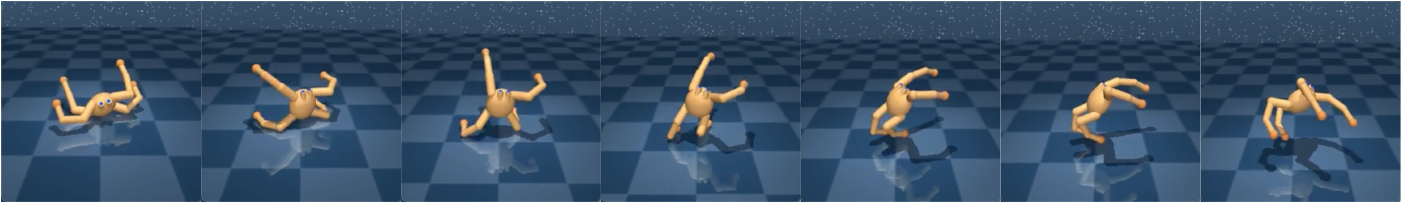}}
\vspace{-0.5em}
\caption{An illustration of the rolling skill learned in \emph{Quadruped}.}
\label{fig:vis-skill-main}
\vspace{-1.5em}
\end{center}
\end{figure*}

We compare CeSD to several strong baselines. Specifically, we compare CeSD to (\romannumeral1) the skill discovery methods, including DIAYN \cite{diayn}, SMM \cite{SMM}, and APS \cite{aps}; (\romannumeral2) the entropy-based exploration methods, including APT \cite{apt}, ProtoRL \cite{proto}, and CIC \cite{cic}; and (\romannumeral3) curiosity-driven exploration methods, including ICM \cite{pathak2017curiosity}, RND \cite{burda2018rnd}, and Disagreement \cite{pathak2019disagreement}; (\romannumeral4) the recently proposed BeCL \cite{becl} algorithm that performs contrastive skill discovery. Most implementations of baselines follow URLB \cite{URLB} and the official code of baselines. We refer to Appendix~\ref{app:urlb} for the hyper-parameters and implementation details.

We do not include DISCO-DANCE \cite{guidance-2023} as a baseline since it does not open-source the code. Meanwhile, Dyna-MPC \cite{rajeswar2023mastering} Dyna-MPC is a model-based finetuning method with extrinsic rewards, while our method focuses on unsupervised pertaining. Choreographer \cite{choreographer} is learned in an offline dataset collected by exploration algorithms, while CeSD and baselines are all learned from scratch via exploring the environment. Thus, we do not use Dyna-MPC and Choreographer as baselines. The recently proposed methods like LSD \cite{LSD-2022}, CSD \cite{USD-2023}, and Metra \cite{metra} are evaluated on different benchmarks other than URLB. We tried to re-implement these methods in URLB tasks based on the official code, and the results are given in Appendix~\ref{app:add-compare}.

In the unsupervised training stage, each method is trained for 2M steps with its intrinsic reward. Then we randomly sample a skill as the policy condition and fine-tune the policy for 100K steps in each downstream task for fast adaptation. Rather than choosing the best skill in the fine-tuning stage, we comprehensively evaluate the generalizability of all skills for adaptation in downstream tasks. We run 10 random seeds for each baseline, which results in 11 algorithms $\times$ 10 seeds $\times$ 12 tasks = 1320 runs. We follow reliable \cite{agarwal2021IQM} to evaluate the aggregated statistics, including mean, median, interquartile mean (IQM), and optimality gap (OG) with the 95\% bootstrap confidence interval. The expert score in calculating the metrics is adopted from \citet{URLB}, which is obtained by an expert DDPG agent. According to the results in Figure~\ref{fig:dmc-result}, our CeSD algorithm achieves the best results in the URLB benchmark. Compared to the entropy-based baselines, our method outperforms CIC (with 75.18\% IQM) and achieves 91.05\% IQM. The results show that partition exploration and distribution constraints in CeSD benefit skill learning and lead to more efficient generalization in downstream tasks compared to the global exploration performed in CIC. Compared to the expert score, CeSD achieves the best OG result with 15.47\%, significantly outperforming the previous state-of-the-art BeCL algorithm (with 25.44\% OG). We also report the evaluation scores of all methods in Appendix~\ref{app:dmc-score}.


\begin{figure}[t]
\begin{center}
\centerline{
\includegraphics[width=1.0\columnwidth]{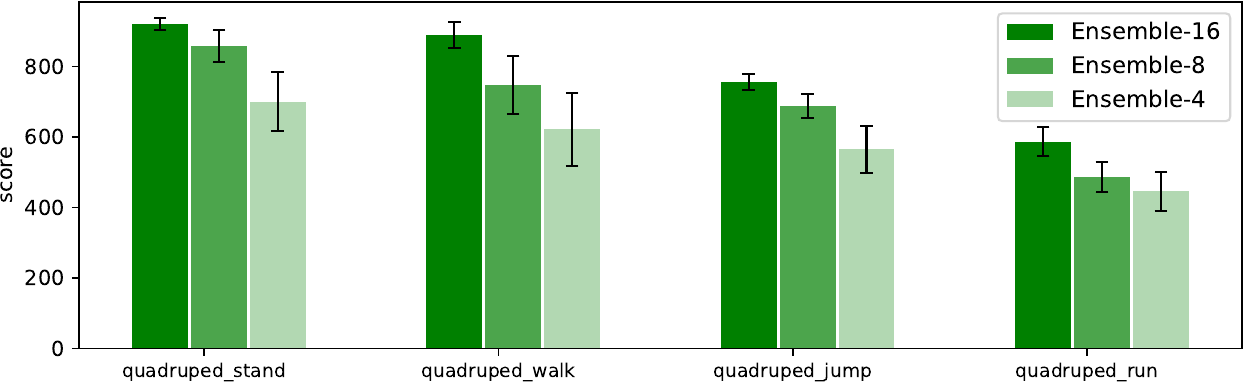}}
\vspace{-0.5em}
\caption{An ablation study of the ensemble skills in \emph{Quadruped}.}
\label{fig:ablation-E}
\vspace{-2em}
\end{center}
\end{figure}

\subsection{Visualization of Skills}

In URLB, we visualize the behaviors of skills learned in the unsupervised training stage. We find many interesting locomotion and manipulation domain skills emerging in the pretraining stage with our method. Specifically, CeSD can learn various locomotion skills, including standing, walking, rolling, moving, somersault, and jumping in \emph{Walker} and \emph{Quadruped}. In \emph{Jaco}, the agent learns various manipulation skills, including moving the arm to explore different areas and controlling the gripper to grasp objects in different locations. An example of rolling skills in the \emph{Quadruped} domain is shown in Figure~\ref{fig:vis-skill-main}. We provide more visualizations of skills in DMC domains in Appendix~\ref{app:visualize}. 

Through partition exploration and distribution constraint, our method learns dynamic and non-trivial behavior during the unsupervised training stage. In contrast, previous skill-discovery methods usually learn distinguishable posing or yoga-style skills but are often static and lack exploration ability, which has been visualized in previous works \cite{cic,becl}. Since our method can learn meaningful behaviors during the unsupervised stage, it obtains superior generalization performance in the fine-tuning stage in various downstream tasks, as shown in Figure~\ref{fig:dmc-result}.

\subsection{Ablation Study} 

In this section, we provide ablation studies on the ensemble number of skills and state distribution constraints in CeSD. 

\begin{figure*}[t]
\begin{center}
\vspace{-0.5em}
\centerline{
\includegraphics[width=1.8\columnwidth]{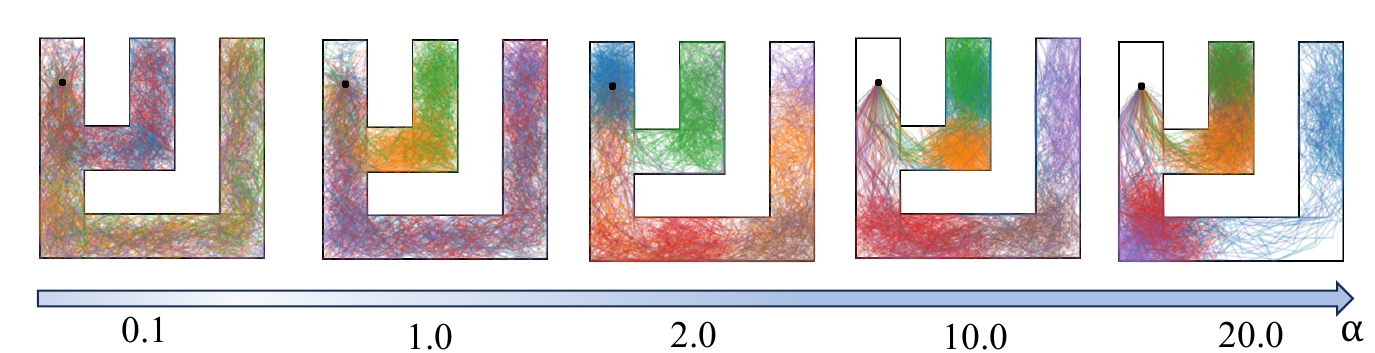}}
\vspace{-0.5em}
\caption{An ablation study of different factors of the regularization reward for distribution constraint in the maze task.}
\label{fig:maze_effect}
\vspace{-2em}
\end{center}
\end{figure*}

\paragraph{Ensemble Value Function} We conduct an ablation study of ensemble value functions in the \emph{Quadruped} domain. We reduce the ensemble number of value functions while keeping the skill number unchanged, which makes a single $Q$-network responsible for the learning of several skills. As shown in Figure~\ref{fig:ablation-E}, as we reduce the ensemble number, the generalization performance of skills also decreases. Since the different skills are enforced to explore independent space, reducing the number of skills will enlarge the exploration space of each value function, reducing the uniqueness of each skill. As an extreme case, reducing the ensemble size to $1$ resembles CIC \cite{cic} that only performs exploration without learning distinguishable skills.


\paragraph{Distribution Constraints} The final intrinsic reward of CeSD is $r^{\rm cesd}_i(s,a) + \alpha \cdot r^{\rm reg}_i(s,a)$, where $r^{\rm reg}_i(s,a)$  performs distribution constraint for different skills. We conduct an ablation study of distribution constraints using different values for $\alpha$, as shown in Figure~\ref{fig:maze_effect}. (\romannumeral1) When $\alpha$ is very small (e.g., $\alpha=0.1$), the trajectories of different skills are mixed and CeSD is very similar to CIC. Nevertheless, since we adopt an ensemble network and partition exploration for different skills, the state coverage of skills also has slight differences. (\romannumeral2) As we increase $\alpha$ (e.g., $\alpha\in[1.0,2.0]$), the regularized reward will force each skill to reduce the visitation probability of states lying in clusters of other skills, which makes the different skills more distinguishable. (\romannumeral3) When the value of $\alpha$ becomes extremely large (e.g., $\alpha\geq 10.0$), the distribution constraints will dominate the reward function, which may hinder the exploration of skills. 


\section{Conclusion}

We have introduced CeSD, a novel skill discovery method assisted by partition exploration and state distribution constraint. We perform self-supervised clustering for collected states and constrain the exploration of each skill based on the assigned cluster, which leads to diverse skills with strong exploration abilities. Extensive experiments in the maze and URLB benchmark show that CeSD can explore complex environments and obtain state-of-the-art performance in adaptation to various downstream tasks. The main limitation of our method is that the ensemble value functions cannot be generalized to the continuous skill space. A future direction is to adopt a randomized value function \cite{azizzadenesheli2018efficient} or hyper $Q$-network \cite{hyperdqn} for implicit ensembles with an infinite number of networks.

\section*{Acknowledgements}

This work was supported by the National Natural Science Foundation of China (No.62306242).

\section*{Impact Statement}
This paper presents work whose goal is to advance the field of Machine Learning. There are many potential societal consequences of our work, none of which we feel must be specifically highlighted here.

\nocite{langley00}

\bibliography{main}
\bibliographystyle{icml2024}

\newpage
\appendix
\onecolumn
\section{Theoretical Analysis}
\label{app:theory}

\subsection{Proof of Theorem \ref{thm:entropy}}
\begin{theorem*}[Restate of Theorem \ref{thm:entropy}]
Let each cluster have the same number of samples, for $i\in[n]$, the relationship between the maximum entropy of $\pi^*$ in the state set $\bS$ and $\pi_i^*$ in the cluster set $\bS_i$ is 
\begin{equation}
H\big(d^{\pi^*}(s)\big)=H\big(d^{\pi^*_i}(s)\big)+C(n),
\end{equation}
where $C(n)=\log n$ depends on the number of clusters $n$.  
\end{theorem*}

\begin{proof}
According to the assumption, each $\bS_i$ should have the same number of samples as $|\bS_i|=\frac{N}{n}$, where we denote the total samples as $|\bS|=N$. For a set of states, the entropy obtains its maximum when the policy uniformly visits each state, as $d^{\pi^\star_i}(s)=\frac{1}{|\bS_i|}=\frac{n}{N}$ in partition exploration, and $d^{\pi^\star}(s)=\frac{1}{|\bS|}=\frac{1}{N}$ in global exploration. 

For policy $\pi^\star_i$ in partition exploration, the corresponding entropy of state distribution is
\begin{equation}
H(d^{\pi^\star_i}(s))
=-\sum_{s\in \bS_i} d^{\pi^\star_i}(s)\log d^{\pi^\star_i}(s)
=-\sum_{s\in \bS_i} \frac{n}{N}\log \frac{n}{N}
=\log N - \log n.
\end{equation}
Similarly, for policy $\pi^\star$ in global exploration, the corresponding entropy of state distribution is
\begin{equation}
H(d^{\pi^\star})
=-\sum_{s\in \bS} d^{\pi^\star}(s)\log d^{\pi^\star}(s)
=-\sum_{s\in \bS} \frac{1}{N}\log \frac{1}{N}
=\log N.
\end{equation}
Then we have the following relationship as
\begin{equation}
H\big(d^{\pi^*}(s)\big)=H\big(d^{\pi^*_i}(s)\big)+C(n),
\end{equation}
where $C(n)=-\log n$ that depends on the number of skills. 
\end{proof}

The primary purpose of Theorem~\ref{thm:entropy} is to relate the entropy of the global optimal policy $\pi^*$ and local optimal policies $\{\pi_i^*\}_{i\in[n]}$, thus showing that maximizing the entropy of each local policy will effectively maximize the entropy of the global policy. In many scenarios where uniform distributions over the state space $\mathbb{S}$ and subsets of state space $\mathbb{S}_i$ are realizable, the maximum entropy policies $\pi^*$ and $\pi^*_i$ are exactly equal to these uniform distributions. Thus we have the relation $H(d^{\pi^*}(s)) = H(d^{\pi^*_i}(s)) + \log n$ for any $i \in [n]$ (as stated in Theorem 3.1). We highlight this fact because it provides a simple yet clear insight into why performing partition exploration in clusters also maximizes global state coverage.

Meanwhile, there are scenarios where uniform distributions are not realizable. In this case, let $d(s)$ be a distribution over $\bS$ that is composed by the local state distributions $\{d^{\pi_i^*}(s)\}_{i\in[n]}$, such that 
$$d(s) = \alpha_i \cdot d^{\pi_i^*}(s) \ \text{ if and only if } \ s \in \bS_i,$$ where $\sum_{i \in [n]} \alpha_i = 1$ and each $\alpha_i$ represents the probability that a state belongs to $\bS_i$. Then, we have 
\begin{align}
    H(d(s)) = \left(\sum_{i=1}^n \alpha_i \cdot H(d^{\pi_i^*}(s)) \right) + H(\{\alpha_1, \alpha_2, \ldots, \alpha_n\}), \label{eq:8}
\end{align}
where $H(\{\alpha_1, \alpha_2, \ldots, \alpha_n\})$ is the entropy of the probability vector $(\alpha_1, \alpha_2, \ldots, \alpha_n)$.
Thus, increasing/maximizing each local entropy $H(d^{\pi_i^*}(s))$ will lead to an increase of the global entropy of the distribution $d(s)$. Eqn.~\eqref{eq:8} is also consistent with our current Theorem 3.1, where $H(d^{\pi_i^*}(s)) = \log N - \log n$ is the entropy of the uniform distribution over $\bS_i$, $H(d(s)) = \log N$ is the entropy of the uniform distribution over $\bS$, and $\alpha_i = 1/n$ for all $i\in[n]$ (i.e., $H(\{\alpha_1, \alpha_2, \ldots, \alpha_n\}) = \log n)$.

Although we do not know the maximum entropy policy $\pi^*$, but it  must satisfy 
\begin{align}
H(d^{\pi^*}(s)) \ge \max_{\alpha_i \in [0,1], \ \sum_{i=1}^n \alpha_i = 1} \left(\sum_{i=1}^n \alpha_i \cdot H(d^{\pi_i^*}(s)) \right) + H(\alpha_1, \alpha_2, \ldots, \alpha_n).
\end{align}

\textbf{Proof of~\eqref{eq:8}}: 
\begin{align}    
H(d(s)) &= -\sum_{s\in \bS}d(s) \log d(s) \\
&= - \sum_{i=1}^n \sum_{s \in \bS_i}d(s) \log d(s) \\
&= - \sum_{i=1}^n \sum_{s \in \bS_i} \alpha_i d^{\pi_i^*}(s) \cdot \log (\alpha_i d^{\pi_i^*}(s)) \\
&= - \sum_{i=1}^n \sum_{s \in \bS_i} \alpha_i d^{\pi_i^*}(s) \cdot [\log ( d^{\pi_i^*}(s)) + \log (\alpha_i)] \\
&= \left(\sum_{i=1}^n -\alpha_i \sum_{s \in \bS_i} d^{\pi_i^*}(s) \log d^{\pi_i^*}(s) \right) + \left(\sum_{i=1}^n -\alpha_i \sum_{s \in \bS_i} d^{\pi_i^*}(s) \log(\alpha_i) \right) \\
&= \sum_{i=1}^n \alpha_i H(d^{\pi_i^*}(s)) + \sum_{i=1}^n -\alpha_i \log (\alpha_i) \\
&= \sum_{i=1}^n \alpha_i H(d^{\pi_i^*}(s)) + H(\{\alpha_1, \alpha_2, \ldots, \alpha_n \}).
\end{align}

This inequality shows that maximizing the entropy of each local policy $H(d^{\pi_i^*}(s))$ will maximize the lower bound on $H(d^{\pi^*}(s))$, which effectively maximizes global state coverage.

\subsection{Proof of Lemma \ref{lemma:distribution}}

\begin{lemma*}[Restate of Lemma \ref{lemma:distribution}]
The divergence between state distribution is bounded on the average divergence of policies $\hat{\pi}_i$ and $\pi_i$, as
\begin{equation}
D_{\rm TV}\big(d^{\hat{\pi}_i}\|d^{\pi_i}\big)
\leq \frac{\gamma}{1-\gamma}\EE_{s\sim d^{\pi_i}}\big[D_{\rm TV}(\hat{\pi}_i(\cdot|s)\|\pi_i(\cdot|s))\big],
\end{equation}
where $D_{\rm TV}(\cdot\|\cdot)$ is the total variation distance.
\end{lemma*}

\begin{proof}
In our proof, we consider finite MDPs, although we can apply the divergence minimizing algorithm for large-scale MDPs. We recall the definition of discounted future state distribution as follows, 
\begin{equation}
d^{\pi} (s) = (1-\gamma)\sum_{t=0}^\infty \gamma^t \rP^{t}_\pi(s).
\end{equation}
We take a vector form of $d^{\pi}$ in a given state set $\bS$, then $\rP^{t}_\pi\in \RR^{|\bS|}$ denotes a vector with components $\rP^{t}_\pi(s)=P(s_t=s|\pi)$. Further, we denote $\rP_\pi\in \RR^{|\bS|\times |\bS|}$ as the transition matrix from $s$ to $s'$ with components $P_\pi(s'|s)=\int P(s'|s,a)\pi(a|s) da$. Then we have 
\begin{equation}
\rP^{t}_\pi=\rP_\pi \rP^{t-1}_\pi=(\rP_\pi)^2 \rP^{t-2}_\pi=\ldots=(\rP_\pi)^t \mu,
\end{equation}
where the $\mu\in\RR^{|\bS|}$ is the initial state distribution. Then we can derive the vector form of state distribution as
\begin{equation}
d^\pi=(1-\gamma) \sum_{t=0}^\infty (\gamma \rP_\pi)^t \mu=(1-\gamma)(I-\gamma \rP_\pi)^{-1} \mu.
\end{equation}

Then we build the relationship between $d^{\hat{\pi}_i}$ and $d^{\pi_i}$, we have
\begin{equation}
\label{eq:d-diff}
d^{\hat{\pi}_i}-d^{\pi_i}=(1-\gamma)\big((I-\gamma \rP_{\hat{\pi}_i})^{-1}-(I-\gamma \rP_{\pi_i})^{-1}\big)\mu
\end{equation}
In the following, we denote $\hG\triangleq (I-\gamma \rP_{\hat{\pi}_i})^{-1}$ and $G\triangleq (I-\gamma \rP_{\pi_i})^{-1}$, then we have 
\begin{equation}
\label{eq:app-hG}
\begin{aligned}
\hG-G
&=\hG(G^{-1}-\hG^{-1})G=\hG(I-\gamma \rP_{\pi_i}-I-\gamma \rP_{\hat{\pi}_i})G\\
&=\gamma \hG(\rP_{\hat{\pi}_i}-\rP_{\pi_i})G.
\end{aligned}
\end{equation}
Plugging \eqref{eq:app-hG} into \eqref{eq:d-diff}, we obtain
\begin{equation}
d^{\hat{\pi}_i}-d^{\pi_i}=\gamma \hG(\rP_{\hat{\pi}_i}-\rP_{\pi_i})(1-\gamma)G\mu=\gamma \hG(\rP_{\hat{\pi}_i}-\rP_{\pi_i})d^{\pi_i}.
\end{equation}
where $d^{\pi_i}=(1-\gamma)G\mu$. Then, we can bound the $L_1$-norm of $d^{\hat{\pi}_i}-d^{\pi_i}$ as
\begin{equation}
\label{eq:app-dbound}
\begin{aligned}
\|d^{\hat{\pi}_i} - d^{\pi_i}\|_1 &= \gamma \|\hG (\rP_{\hat{\pi}_i}-\rP_{\pi_i}) d^{\pi_i}\|_1 \\
&\leq \gamma \|\hG\|_1 \|(\rP_{\hat{\pi}_i}-\rP_{\pi_i}) d^{\pi_i}\|_1.
\end{aligned}
\end{equation}
In \eqref{eq:app-dbound}, the first term $\|\hG\|_1$ is bounded by
\begin{equation}
\begin{aligned}
\big\|\hG\big\|_1 = \big\|(I - \gamma \rP_{\hat{\pi}_i})^{-1}\big\|_1 \leq \sum_{t=0}^{\infty} \gamma^t \left\|\rP_{\hat{\pi}_i}\right\|_1^t = \frac{1}{1-\gamma}.
\end{aligned}
\end{equation}
The second term $\|(\rP_{\hat{\pi}_i}-\rP_{\pi_i}) d^{\pi_i}\|_1$ of \eqref{eq:app-dbound} is bounded by
\begin{equation}
\begin{aligned}
\|(\rP_{\hat{\pi}_i}-\rP_{\pi_i}) d^{\pi_i}\|_1 &= \sum_{s'} \left| \sum_s \big(P_{\hat{\pi}_i}(s'|s)-P_{\pi_i}(s'|s)\big) d^{\pi_i}(s) \right| 
\leq \sum_{s,s'} \left| \Big(P_{\hat{\pi}_i}(s'|s)-P_{\pi_i}(s'|s)\Big)\right| d^{\pi_i}(s) \\
&= \sum_{s,s'} \left| \sum_a P(s'|s,a) \left(\hat{\pi}_i(a|s) - \pi_i(a|s) \right)\right| d^{\pi_i}(s) 
\leq \sum_{s,a,s'} P(s'|s,a) \Big|\hat{\pi}_i(a|s) - \pi_i(a|s)\Big| d^{\pi_i}(s) \\
&\leq \sum_{s,a} \Big|\hat{\pi}_i(a|s) - \pi_i(a|s) \Big| d^{\pi_i}(s) 
= 2 \EE_{s \sim d^{\pi_i}} \left[ D_{TV} (\hat{\pi}_i\|\pi_i)[s] \right].
\end{aligned}
\end{equation}
Then, we obtain
\begin{equation}
\begin{aligned}
D_{\rm TV}\big(d^{\hat{\pi}_i}\|d^{\pi_i}\big)&=\frac{1}{2}\|d^{\hat{\pi}_i} - d^{\pi_i}\|_1\leq \frac{1}{2}\gamma \frac{1}{1-\gamma} 2 \EE_{s \sim d^{\pi_i}} \left[ D_{TV} (\hat{\pi}_i\|\pi_i)[s] \right]\\
&=\frac{\gamma}{1-\gamma}\EE_{s \sim d^{\pi_i}} \left[ D_{TV} (\hat{\pi}_i\|\pi_i)[s]\right].
\end{aligned}
\end{equation}
\end{proof}

\subsection{Proof of Theorem~\ref{thm:distance}}

\begin{theorem*}[Restate of Theorem~\ref{thm:distance}]
Assuming the distance between state distribution is bounded by $D_{\rm TV}\big(d^{\hat{\pi}_i}\|d^{\pi_i}\big)\leq \delta$, the entropy difference between state distribution can be bounded by
\begin{equation}
\big|H(d^{\hat{\pi}_i})-H(d^{\pi_i})\big|\leq \delta \log\big(|\bS^{\rm pe}_i|-1\big)+h(\delta).
\end{equation}
where $h(x) := -x \log(x) - (1-x) \log (1-x)$ is the binary entropy function.
\end{theorem*}

\begin{proof}
We first introduce two random variables $X$ and $Y$ on the state space $\bar{\mathcal{S}}=\bS^{\rm pe}_i$, such that $X$ follows from the distribution $d^{\hat{\pi}_i}$, while $Y$ follows from the distribution $d^{\pi}_i$. Next,  we use a \emph{coupling technique} to construct a joint probability distribution of $X$ and $Y$, denoted by $g_{XY} \in \Delta(\bar{\mathcal{S}} \times \bar{\mathcal{S}})$, such that: 
\begin{enumerate}
    \item The marginal distributions of $g_{XY}$, denoted by $g_X$ and $g_Y$, satisfy $g_X = d^{\hat{\pi}_i}$ and $g_Y = d^{\pi_i}$ respectively. More precisely, we have
    \begin{align}
    &g_X(s) :=\sum_{s' \in \s}g_{XY}(s,s') = d^{\hat{\pi}_i}(s), \quad \forall s \in \bar{\mathcal{S}}, \label{eq:app-thmq1}\\
    & 
    g_Y(s) := \sum_{s' \in \s}g_{XY}(s',s) = d^{\pi_i}(s), \quad \forall s \in \bar{\mathcal{S}}. \label{eq:app-thmq2}
    \end{align} 
    \item For all $s \in \bar{\mathcal{S}}$, $g_{XY}(s, s) = \min\{d^{\hat{\pi}_i}(s), d^{\pi_i}(s) \}$. 
\end{enumerate}
For $s, s' \in \bar{\mathcal{S}}$ such that $s \ne s'$, we can choose the value of $g_{XY}(s, s')$ arbitrarily, as long as the conditions in~\eqref{eq:app-thmq1} and~\eqref{eq:app-thmq2} are satisfied. Based on this joint probability distribution $g_{XY}$, we can calculate the probability that $X$ does not equal $Y$:
\begin{align}
\Pr(X \ne Y) &= 1 - \sum_{s \in \s} g_{XY}(s,s) \notag \\
&= 1-  \sum_{s \in \bar{\mathcal{S}}} \min\{d^{\hat{\pi}_i}(s), d^{\pi_i}(s) \} \notag\\
&= \frac{1}{2}\left[\sum_{s\in \bar{\mathcal{S}}} d^{\hat{\pi}_i}(s) + d^{\pi_i}(s) - 2\min\{d^{\hat{\pi}_i}(s),d^{\pi_i}(s) \}  \right] \notag\\
&= \frac{1}{2} \sum_{s \in \bar{\mathcal{S}}} |d^{\hat{\pi}_i}(s) - d^{\pi_i}(s)| \notag \\
&= D_{\mathrm{TV}}(d^{\hat{\pi}_i}\|d^{\pi_i}). \label{eq:app-tv}
\end{align}
Since the random variable $X$ follows from the distribution $d^{\hat{\pi}_i}$, the entropy of $X$, denoted by $H(X)$, is equivalent to the entropy $H(d^{\hat{\pi}_i})$.  Similarly, the entropy $H(Y)$ is equivalent to $H(d^{\pi_i})$.
Using standard information-theoretic inequalities, we have 
\begin{align}
\begin{cases}
\big|H\big(d^{\hat{\pi}_i}\big) - H\big(d^{\pi_i}\big)\big| = H(X) - H(Y) \le H(X) - I(X;Y) = H(X|Y), \quad \text{if } H(X) \ge H(Y); \\
\big|H\big(d^{\hat{\pi}_i}\big) - H\big(d^{\pi_i}\big)\big| = H(Y) - H(X) \le H(Y) - I(X;Y) = H(Y|X), \quad \text{if } H(X) < H(Y);
\end{cases}
\end{align}
where $I(X;Y)$ is the \emph{mutual information} of $X$ and $Y$ (with respect to the joint probability distribution $g_{XY}$), and $H(X|Y)$ and $H(Y|X)$ denote the \emph{conditional entropy}. Applying Fano's inequality yields that
\begin{align}
H(X|Y) \le   \Pr(X \ne Y) \cdot \log(|\bar{\s}|-1)  + h(\Pr(X \ne Y) ), \\
H(Y|X) \le   \Pr(X \ne Y) \cdot \log(|\bar{\s}|-1)  + h(\Pr(X \ne Y) ).
\label{eq:app-fano}
\end{align}
Combining Eqns.~\eqref{eq:app-tv}-\eqref{eq:app-fano}, we eventually obtain that
\begin{align}
    |H(d^{\hat{\pi}_i}) - H(d^{\pi_i})| \le D_{\mathrm{TV}}\big(d^{\hat{\pi}_i}\|d^{\pi_i}\big) \cdot \log(|\bar{\s}|-1)  + h(D_{\mathrm{TV}}\big(d^{\hat{\pi}_i}\|d^{\pi_i})\big),
\end{align}
which concludes our proof under the assumption that $D_{\mathrm{TV}}(d^{\hat{\pi}_i}\|d^{\pi_i})\leq \delta$.
\end{proof}

Theorem~\ref{thm:distance} shows that if the distance between two probability distributions is bounded, then their entropy difference can also be bounded. A similar proof of Theorem~\ref{thm:distance} was first appeared in \citet{zhang2007estimating} and \citet{csiszar2011information} (Ex 3.10). The proof relies on a coupling technique (used to relate two random variables), standard information-theoretic inequalities in \citet{cover1999elements} (Sec 2.3), and Fano's inequality in \citet{cover1999elements} (Sec 2.10).

\newpage
\section{Additional Experiments in Maze}\label{app:maze}

\subsection{Tree Map}

We conducted an additional experiment on a Tree-like map. This map is more challenging since the agent needs to explore the deepest branches to maximize the state coverage. According to Figure~\ref{fig:tree_maze}, empowerment-based methods can learn distinguishable skills while having limited exploration ability in the tree map. In contrast, CIC obtains a well global state coverage while the trajectories of different skills are indistinguishable. CeSD can also reach the deepest branches and overcome the limitations of CIC. Specifically, CeSD generates distinguishable skills, where the different skills perform independent exploration and have fewer overlapping visitation areas.

\begin{figure*}[h!]
\begin{center}
\centerline{
\includegraphics[width=1.0\columnwidth]{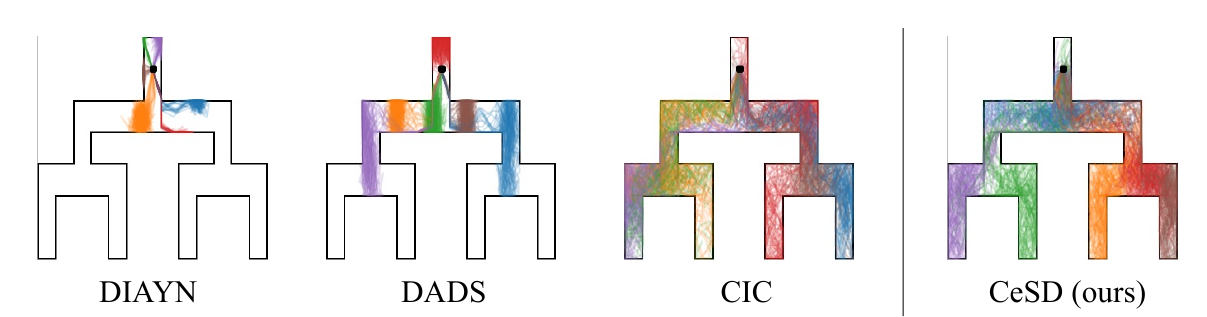}}
\vspace{-0.5em}
\caption{The visualization of skill discovery methods in a Tree-like maze. Different colors represent the state trajectories with different skill vectors. We let the agent start moving from the black dot in the upper corner and sample 20 trajectories for each skill for visualization. Our method can explore the deepest branches and also learn distinguishable skills.}
\label{fig:tree_maze}
\vspace{-2em}
\end{center}
\end{figure*}

\subsection{The Comparison of MI and Entropy Estimation}

We compare the mutual information (MI) between states and skills  (i.e., $I(S;Z)$) and the state entropy (i.e., $H(d^{\pi}(s))$) of the final policies in different methods, where the $d^{\pi}(s)$ is estimated by generating trajectories from all skills $\{\pi_i\}_{i\in[n]}$. 

To estimate the MI term, we generate several trajectories for each learned skill and perform MI estimation using the MINE \cite{mine-2018} estimator. MINE adopts a score function $T:\cS\times\cZ\rightarrow \RR$ represented by a neural network in estimation. The joint samples come from the joint distribution $(s,z)\sim P_{S,Z}$, where the states are generated by the corresponding skills. $\bar{s}\sim d^{\pi}_s$ and  $\bar{z}\sim P_Z$ are sampled from the corresponding marginal distribution. Then the MINE estimation given as: $\sup_{T\in\cF}\EE_{P_{S,Z}}[T(s,z)]-\log (\EE_{P_S\otimes P_Z}[e^{T(s,z)}])$, where $\cF$ is the function class. In addition, we perform entropy estimation by using the particle-based entropy estimator \cite{apt}, which is the same as in our method. 

As shown in Figure~\ref{fig:entropy-mine}, CIC obtains much lower MI than other skill discovery methods but obtains the largest state entropy. CeSD can balance state coverage and empowerment via partition exploration and distribution constraints, which leads to diverse skills and also has better state coverage than previous skill discovery algorithms.

\begin{figure}[h]
\begin{center}
\centerline{
\includegraphics[width=0.8\columnwidth]{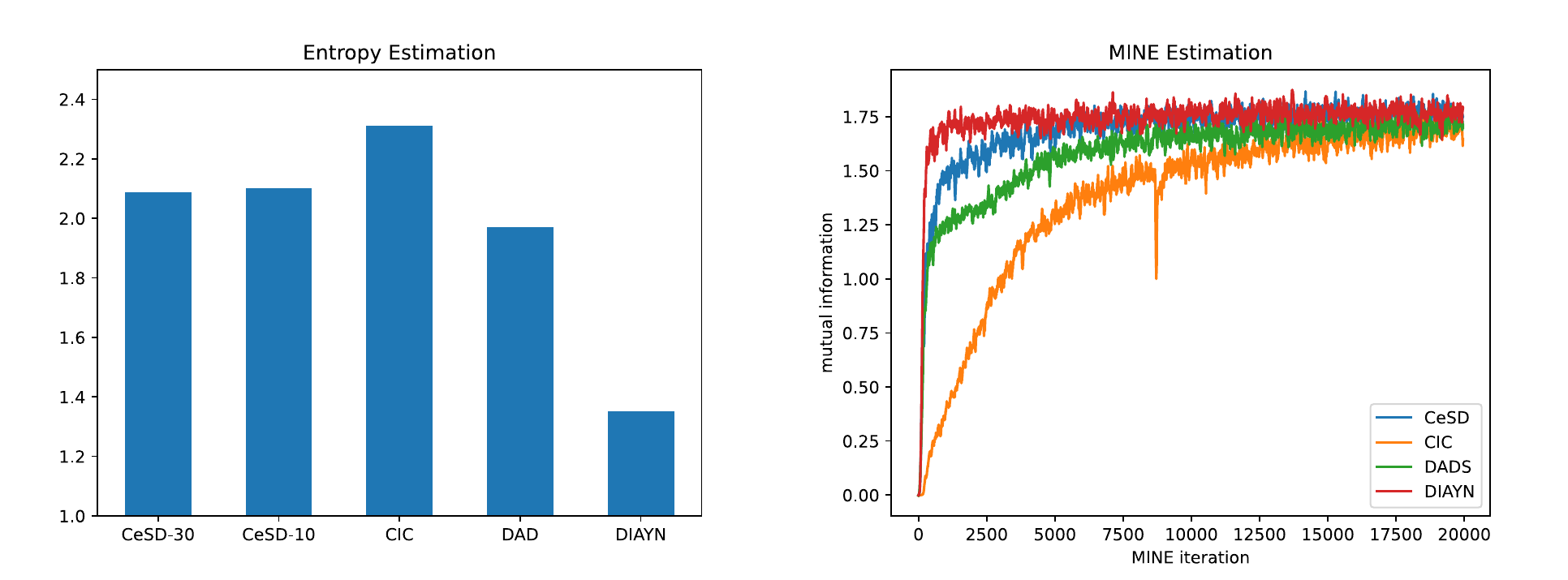}}
\vspace{-0.5em}
\caption{The qualitative result for the mutual information estimation and the entropy estimation in maze.}
\vspace{-1em}
\label{fig:entropy-mine}
\end{center}
\end{figure}

\newpage
\section{Additional Experiments in URLB}\label{app:urlb}

\subsection{Implementation Details and Hyperparameters}

We introduce the implementation details of the proposed CeSD algorithm as follows. 
(\romannumeral1) \textbf{Skills.} In the pertaining stage, a skill vector $z_i$ is sampled from a $n$-dimensional discrete distribution following a uniform distribution every fixed time-steps. The agent interacts with the environment based on $\pi(a|s,z_i)$, and the obtained transition $(s,a,r,s',z_i)$ is stored in a replay buffer. We use $n=16$ skills in all tasks. 
(\romannumeral2) \textbf{Clustering}. In training, we sample a batch of transitions $\{(s,a,r,s',z_i)\}$ from the replay buffer and perform clustering for the states based on the prototypes. The encoder network $f_{\theta}(s)$ of states is an MLP network with ${\rm obs\_dim}\rightarrow1024\rightarrow1024\rightarrow1024$ architecture and ReLU activations. The output of $f_{\theta}(s)$ is the same dimensions as the prototype $c_i\in\RR^{m}$, where we use $m=16$ in experiments. We perform a coarse search for prototype updates per iteration for each domain from $\{4,5,6\}$. The temperature value in Eq.~\eqref{eq:cluster} is set to $\tau=0.1$. The prototypes are trained using stochastic gradient optimization with Adam with a learning rate of $10^{-4}$. 
(\romannumeral3) \textbf{Entropy estimation}. We perform particle estimation on the feature space and $k$ is set to 16. The entropy estimation is performed in each cluster independently based on the prototypes.
(\romannumeral4) \textbf{Ensemble value functions}. Each $Q$-network is a MLP with ${\rm obs\_dim}+{\rm action\_dim}\rightarrow 512\rightarrow1024\rightarrow1024\rightarrow1$. In practice, we use the vectorized linear layers in critic for parallel inference of the ensemble $Q$-network. The ensemble size of the critic is the same as the number of skills. We denote the ensemble $Q$-value for a batch with $b$ samples as $\mathbf{Q}\in\RR^{n\times b}$. Then we adopt a mask matrix $\mathbf{M}\in\RR^{n\times b}$ where each column is a one-hot vector $[0,\ldots,1,\ldots,0]$ that denotes the cluster index of this transition by following $\mathbf{p}^{(t)}$. Then the masked value function is calculated by $\mathbf{Q}\odot\mathbf{M}$ for the TD-error calculation. 
(\romannumeral5) \textbf{Policy learning}. The policy network is an MLP with $\rm obs\_dim+skill\_dim\rightarrow 50\rightarrow 1024\rightarrow 1024\rightarrow action\_dim$ architecture, where we use the same actor architecture for CeSD and other baselines. (\romannumeral6) \textbf{Intrinsic reward}. For practical implementation, the state set $\bS^{\rm pe}$ and $\bS^{\rm clu}$ used in intrinsic reward is calculated in batches rather than all collected states. We use the batch size of 1024 in all methods. 

We adopt DDPG as the basic algorithm in policy training for all baselines. Table \ref{table:common_hyperparams} summarizes the hyperparameters of our method and the basic DDPG algorithm. We refer to our released code for the details.

\begin{table}[h]
\caption{\label{table:common_hyperparams} Hyper-parameters for CeSD and the basic DDPG algorithm for all methods.}
\vspace{0.5em}
\centering
\begin{tabular}{lc}
\hline
\textbf{BeCL hyper-parameter}       & \textbf{Value} \\
\hline

\: skill dim / ensemble size $n$ & 16 discrete \\
\: prototype dim $m$ & 16 \\
\: prototype update iterations in clustering & \{4, 5, 6\}\\
\: temperature $\kappa$ for clustering & 0.1 \\
\: $k$-nearest-neighbor in particle estimation & 16 \\
\: skill sampling frequency (steps) & 50 \\
\hline
\textbf{DDPG hyper-parameter}       & \textbf{Value} \\
\hline
\: replay buffer capacity & $10^6$ \\
\: action repeat & $1$ \\
\: seed frames & $4000$ \\
\: $n$-step returns & $3$ \\
\: mini-batch size & $1024$  \\
\: seed frames & $4000$ \\
\: discount ($\gamma$) & $0.99$ \\
\: optimizer & Adam \\
\: learning rate & $10^{-4}$ \\
\: agent update frequency & $2$ \\
\: critic target EMA rate ($\tau_Q$) & $0.01$ \\
\: features dim. & $1024$  \\
\: hidden dim. & $1024$ \\
\: exploration stddev clip & $0.3$ \\
\: exploration stddev value & $0.2$ \\
\: number of pre-training frames &  $2\times 10^6$ \\
\: number of fine-tuning frames & $1 \times 10^5$ \\
\hline
\end{tabular}
\end{table}

\subsection{Algorithmic Description}

We give algorithmic descriptions of the pretraining and finetuning stages in Algorithm~\ref{app:pretrain_algo} and Algorithm~\ref{app:finetune_algo}, respectively. We evaluate the adaptation efficiency of BeCL following the pretraining and finetuning procedures in URLB. Specifically, in the pretraining stage, latent skill $z$ is changed and sampled from a discrete distribution $p(z)$ in every fixed step and the agent interacts with the environments based on $\pi_\theta(a|s,z)$. In the finetuning stage, a skill is randomly sampled and keep fixed in all steps. The actor and critic are updated by extrinsic reward after first 4000 steps.

In our experiments of the \emph{Walker} domain, pretraining one seed of CeSD for 2M steps takes about 11 hours while fine-tuning downstream tasks for 100k steps takes about 20 minutes with a single 4090 GPU. 

\begin{algorithm}[h]
\caption{Unsupervised Pretraining of CeSD}
\label{app:pretrain_algo}
\begin{algorithmic}
  \STATE {\bfseries Input:} number of pretraining frames $N_{PT}$, skill dimension $n$, batch size $N$, and skill sampling frequency $N_{\rm update}$.
  \STATE \textbf{Initialize} the environment, random actor $\pi_\psi(a | s,z)$, ensemble $Q$-network $\{Q_{\phi_i}(s,a)\}$ and target network $\{Q_{\phi'_i}(s,a)\}$, state encoder $f_\theta$, the prototype vectors $\{c_1,\ldots,c_n\}$, and replay buffer $\cD$
  \FOR{$t=1$ {\bfseries to} $N_{PT}$} 

    \STATE Randomly choose $z_i$ \text{from category distribution} $p(z)$ every $N_{\rm update}$ steps.
    \STATE Interact with environment $\tau_{z_i}\sim \pi_\psi(a | s, z_i), ~p(s'|s, a)$ and store the transitions to buffer $\cD$.
    \IF{$t \geq t_0$}
    \STATE Sample a batch a transitions from $\cD$ : $\{(s_i,a_i,s_i^{\prime},z_i)\}_{i\in[N]}$.
    \STATE Calculate $\mathbf{p}^{(t)}$ for each transitions based on prototypes via Eq.~\eqref{eq:cluster} and obtains the cluster index as $\{\hat{z}_i\}_{i\in [N]}$. 
    \STATE Perform particle estimation in each cluster and calculate the entropy-based intrinsic reward $\{r^{\rm cesd}_i\}_{i\in[N]}$.
    \STATE Calculate the constraint reward $\{r^{\rm reg}_i\}_{i\in [N]}$ based on the cluster index $\{\hat{z}_i\}$ (i.e., $\bS^{\rm clu}$) and skill label $\{z_i\}$ (i.e., $\bS^{\rm pe}$).
    \STATE Calculate mask matrix $\mathbf{M}$ and update the ensemble $Q$-network with $\{(s_i,a_i,s_i^{\prime},z_i)\}_{i\in[N]}$ and $\{r^{\rm cesd}_i+\alpha\cdot r^{\rm reg}_i\}_{i\in [N]}$.
    \STATE Update the policy network $\pi_\psi(a | s,z_i)$ by maximizing the corresponding critic function $Q_{\phi_i}(s,a)$.
    \ENDIF
  \ENDFOR
\end{algorithmic}
\end{algorithm}
\vspace{-1em}

\begin{algorithm}[h]
\caption{Downstream Finetuning of CeSD} 
\label{app:finetune_algo}
\begin{algorithmic}
  \STATE {\bfseries Input: }actor $\pi_\psi(a | s,z)$ and critic $Q_\phi(s,a)$ with weights from the pretraining phase, randomly sampled a skill vector $z^\star$ from $p(z)$, and the number of finetuning frames $N_{FT}$ batch size $N$. Initialized environment and replay buffer $\cD$.
  \FOR{$t=1$ {\bfseries to} $N_{FT}$} 

    \STATE Choose the action by $a_t \sim \pi_\theta(a | s_t,z^\star)$.
    \STATE Interact with environment to obtain $s_{t+1}, r_t$ with extrinsic reward from downstream task.
    \STATE Store $(s_{t}, a_t, s_{t+1}, r_{t}, z^\star)$ into buffer $\mathcal{D}$.
    \IF{$t \geq 4,000$}
    \STATE Sample a batch $\{(\mathbf{a}^{(i)}, \mathbf{s}^{(i)}, \mathbf{s}^{\prime(i)}, \mathbf{r}^{(i)}, \mathbf{z}^{(i)})\}_{i=1}^{N}$ from the replay buffer $\mathcal{D}$.
    \STATE Update actor $\pi_\theta(a | s,z^\star)$ and critic $Q_\psi(s,a,z^\star)$ using extrinsic reward $r$ in Eq.~\eqref{eq:tderror} and Eq.~\eqref{eq:actor}.
    \ENDIF
  \ENDFOR
\end{algorithmic}
\end{algorithm}

\subsection{Description of Baselines}

A comparison of intrinsic rewards and representations used in unsupervised RL baselines is summarized in Table~\ref{table:baselines}. According to the taxonomy in URLB \cite{URLB}, (\romannumeral1) the \textbf{knowledge-based} baselines adopt curiosity measurements by training an encoder to predict the dynamics, and use the prediction-error of next-state (e.g., ICM \cite{pathak2017curiosity}), prediction variance (e.g., Disagreement \cite{pathak2019disagreement}), or the divergence between a random network prediction (e.g., RND \cite{burda2018rnd}) as intrinsic rewards; (\romannumeral2) the \textbf{data-based} or entropy-based methods estimate the state entropy via particle estimation and use the state entropy estimation as the intrinsic reward in exploration, including as APT \cite{apt}, ProtoRL \cite{proto} and CIC \cite{cic}; (\romannumeral3) the \textbf{competence-based} or empowerment-based baselines aim to learn latent skill $z$ by maximizing the MI between states and skills: $I(S;Z) = H(S) - H(S|Z) = H(Z) - H(Z|S)$. The different methods adopt various variational forms in estimating the MI term, including the forward form in APS \cite{aps} and DADS \cite{dads}, and the reverse form in DIAYN \cite{diayn}. BeCL \cite{becl} is also a competence-based method and adopts a multi-view perspective and maximizes the MI term $I(S^{(1)};S^{(2)})$, where $S^{(1)}$ and $S^{(2)}$ are generated by the same skill. 

We adopt the baselines of open source code implemented by URLB (\url{https://github.com/rll-research/url_benchmark}), CIC (\url{https://github.com/rll-research/cic}), and BeCL (\url{https://github.com/Rooshy-yang/BeCL}). CeSD can be considered as a data-based method, but also has the advantages of competence-based methods in learning diverse skills. We adopt partition exploration with clusters to learn distinguishable skills without MI estimation. More descriptions of the baselines can be found in URLB \cite{URLB}. 

\begin{table}[h]
  \caption{Summary of baseline methods.}
  \label{table:baselines}
  \centering
\resizebox{\columnwidth}{!}{
  \begin{tabular}{llll}
    \toprule

    Name     & Algo. Type     & Intrinsic Reward & Explicit max $H(s)$\\
    \midrule
    ICM~\cite{pathak2017curiosity} & Knowledge  & $ \| f(s_{t+1}|s_t,a_t) - s_{t+1} \|^2$ & No \\
    Disagreement~\cite{pathak2019disagreement}     & Knowledge & $\textrm{Var} \{ f_i(s_{t+1}|s_t,a_t) \} \quad i =1,\dots,N$   & No   \\
    RND~\cite{burda2018rnd}     & Knowledge       &  $\| f(s_t,a_t) - \tilde{f}(s_t, a_t) \|^2_2$ & No \\

    \midrule 
    APT~\cite{apt} & Data & $ \sum_{j\in \textrm{KNN}} \log \| f(s_t) - f(s_j) \| \quad f \in \textrm{random or ICM }$   & Yes \\
    ProtoRL~\cite{proto}     & Data & $\sum_{j\in \textrm{KNN}} \log \| f(s_t) - f(s_j) \| \quad f \in \textrm{prototypes}$     & Yes  \\
     CIC~\cite{cic}  & Data \tablefootnote{The newest NeurIPS version of CIC \url{https://openreview.net/forum?id=9HBbWAsZxFt} has two designs of intrinsic reward including the NCE term and KNN reward, which represent competence-base and data-based designs respectively. Since CIC obtains the best performance in URLB with KNN reward only and NCE is used to update representation, we use KNN reward as its intrinsic reward and consider it as a data-based method in this paper.}	& $ \sum_{j\in \textrm{KNN}} \log \| f(s_t,s'_t) - f(s_j,s'_j) \| \quad f \in \textrm{contrastive} $	& Yes \\     
   \midrule 
     SMM~\cite{lee2019smm}     & Competence       &  $\log p^*(s) -\log q_z(s)  - \log p(z) + \log d(z|s)$ & Yes \\
    DIAYN~\cite{diayn} & Competence  & $  \log q(z|s) - \log p(z)$   & No  \\
    APS~\cite{aps}    & Competence & $r^{\text{APT}}_{t}(s) + \log q(s|z)$   & Yes
    \\
    BeCL~\cite{becl}  &   Competence      & $\exp(f(s^{(1)}_t)^{\top}f(s^{(2)}_t)/\kappa ) / \sum_{s_j \sim S^{-}\bigcup s_t^{(2)}}\exp (f(s_j)^{\top}f(s^{(1)}_t)/\kappa $ & No \\
    \bottomrule
  \end{tabular}}
\end{table}

\subsection{URLB Environments}

An illustration of URLB tasks is given in Figure~\ref{fig:env_urlb}. There are three domains (i.e., \emph{Walker}, \emph{Quadruped}, and \emph{Jaco}), and each domain has four different downstream tasks. The environment is based on DMC \cite{dmc}. The episode lengths for the Walker and Quadruped domains are set to 1000, and the episode length for the Joco domain is set to 250, which results in the maximum episodic reward for the Walker and Quadruped domains being 1000, and for Jaco Arm being 250. 

\begin{figure}[h]
\begin{center}
\centerline{
\includegraphics[width=0.8\columnwidth]{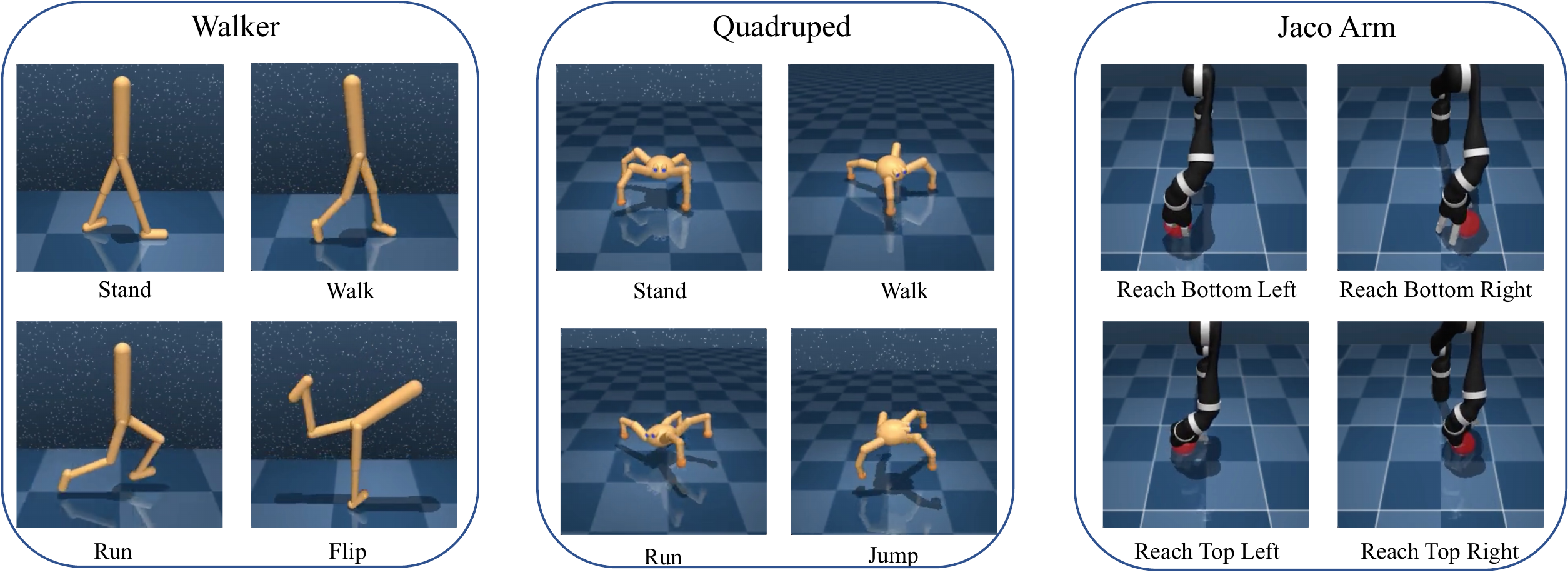}}
\caption{Illustration of domains and downstream tasks in URLB \cite{URLB}. Each domain has four downstream tasks.}
\label{fig:env_urlb}
\end{center}
\end{figure}

\subsection{Visualization of Skills}\label{app:visualize}

As shown in Figure~\ref{fig:vis-skill-walk}, Figure~\ref{fig:vis-skill-quad}, and Figure~\ref{fig:vis-skill-jaco}, we give more visualization results of DMC domains. CeSD can learn various locomotion skills, including standing, walking, rolling, moving, somersault, and jumping in \emph{Walker} and \emph{Quadruped} domains; and also learns various manipulation skills by moving the arm to explore different areas, opening and closing the gripper in different locations in \emph{Jaco} domain. The learned meaningful skills lead to superior generalization performance in the fine-tuning stage of various downstream tasks. 

\begin{figure}[h!]
\begin{center}
\centerline{
\includegraphics[width=1.0\columnwidth]{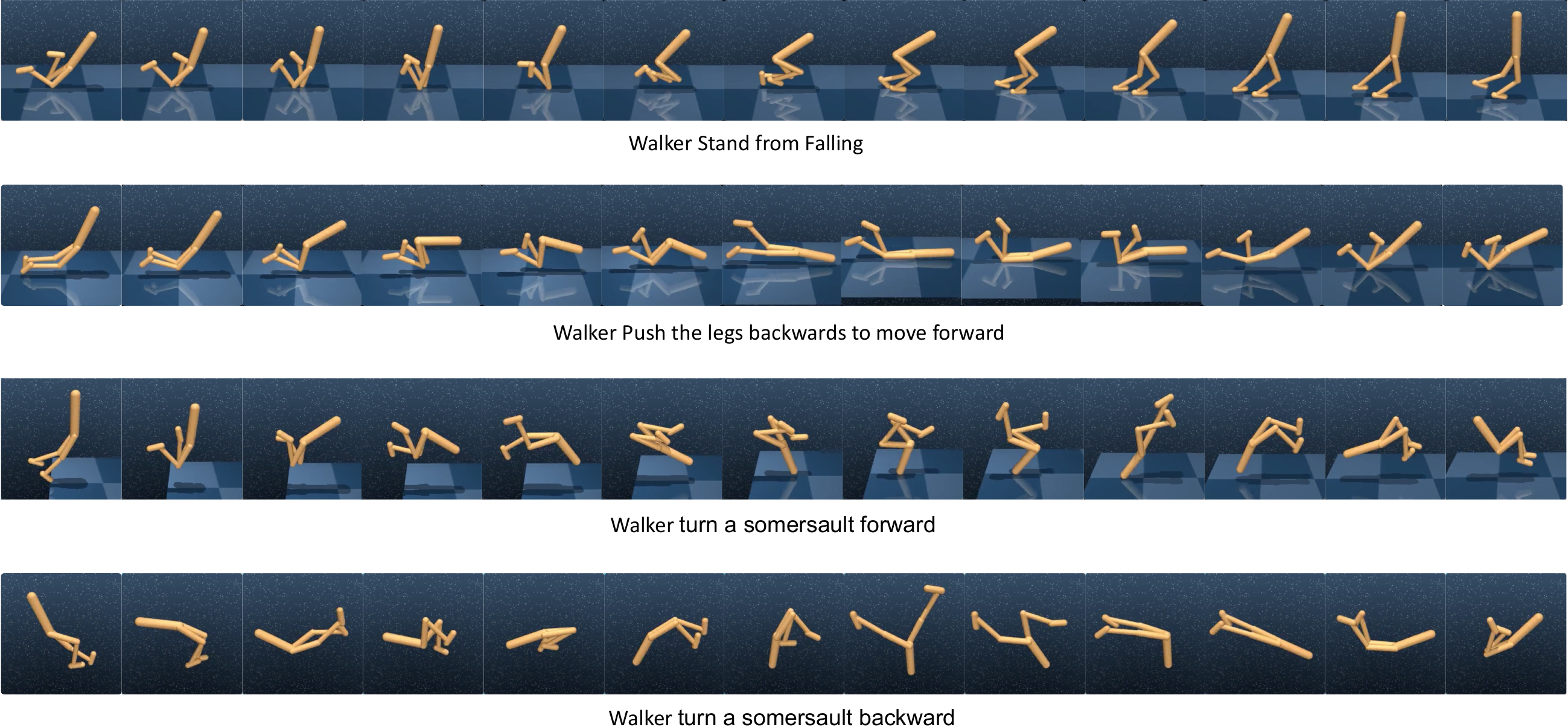}}
\caption{Visualization of representative skills learned of CeSD in the Walker domain. The Walker agent learns some interesting skills like standing and moving. The agent also learns highly difficult skills that turn a somersault forward and backward.  
}
\label{fig:vis-skill-walk}
\end{center}
\end{figure}

\begin{figure}[h!]
\begin{center}
\centerline{
\includegraphics[width=1.0\columnwidth]{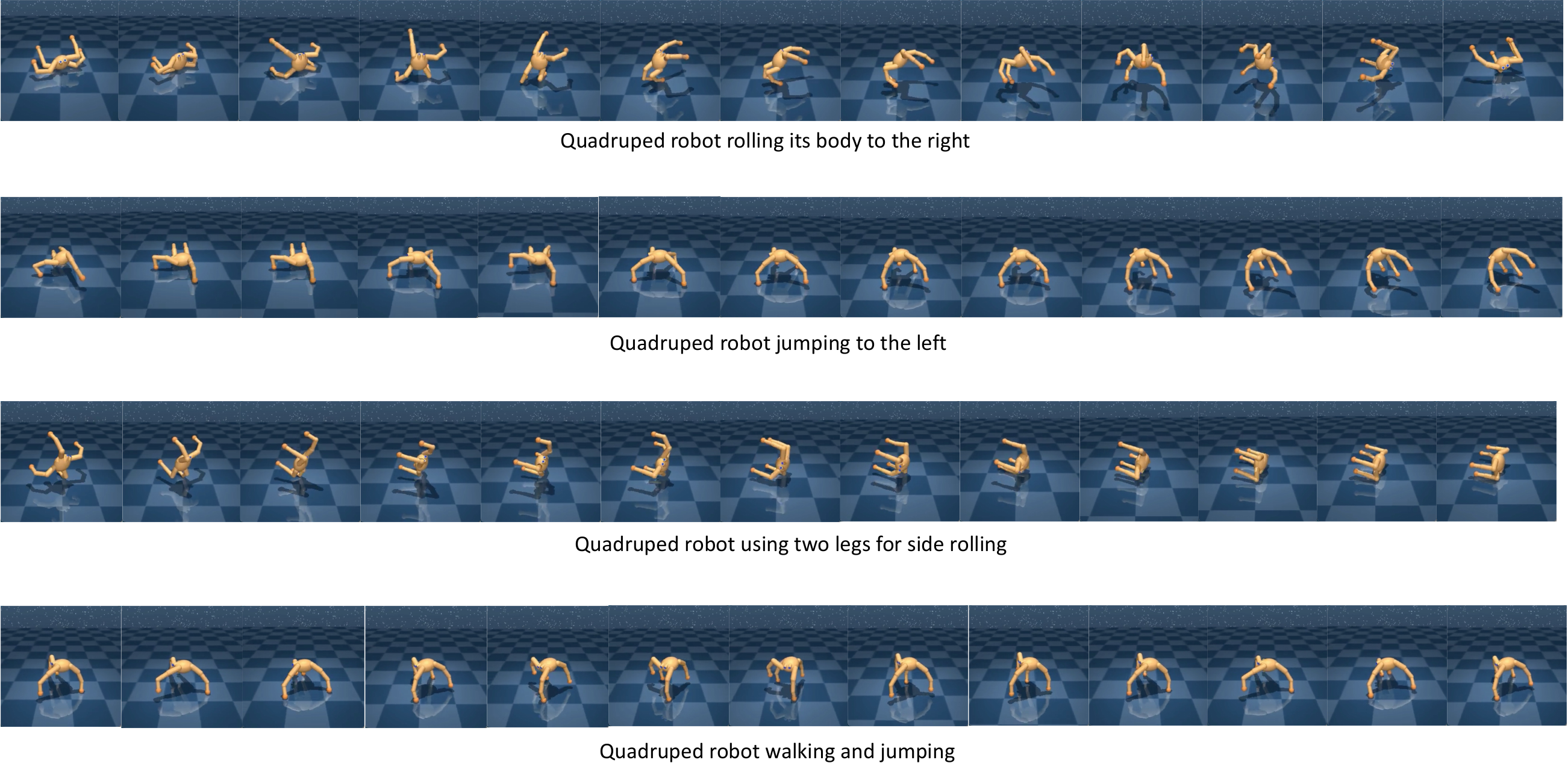}}
\caption{Visualization of representative skills learned of CeSD in the Quadruped domain. The Quadruped agent learns challenging skills like walking, rolling, and jumping that benefit downstream tasks. Also, a novel two-leg rolling skill is learned in pre-training.   
}
\label{fig:vis-skill-quad}
\end{center}
\end{figure}

\begin{figure}[h!]
\begin{center}
\centerline{
\includegraphics[width=1.0\columnwidth]{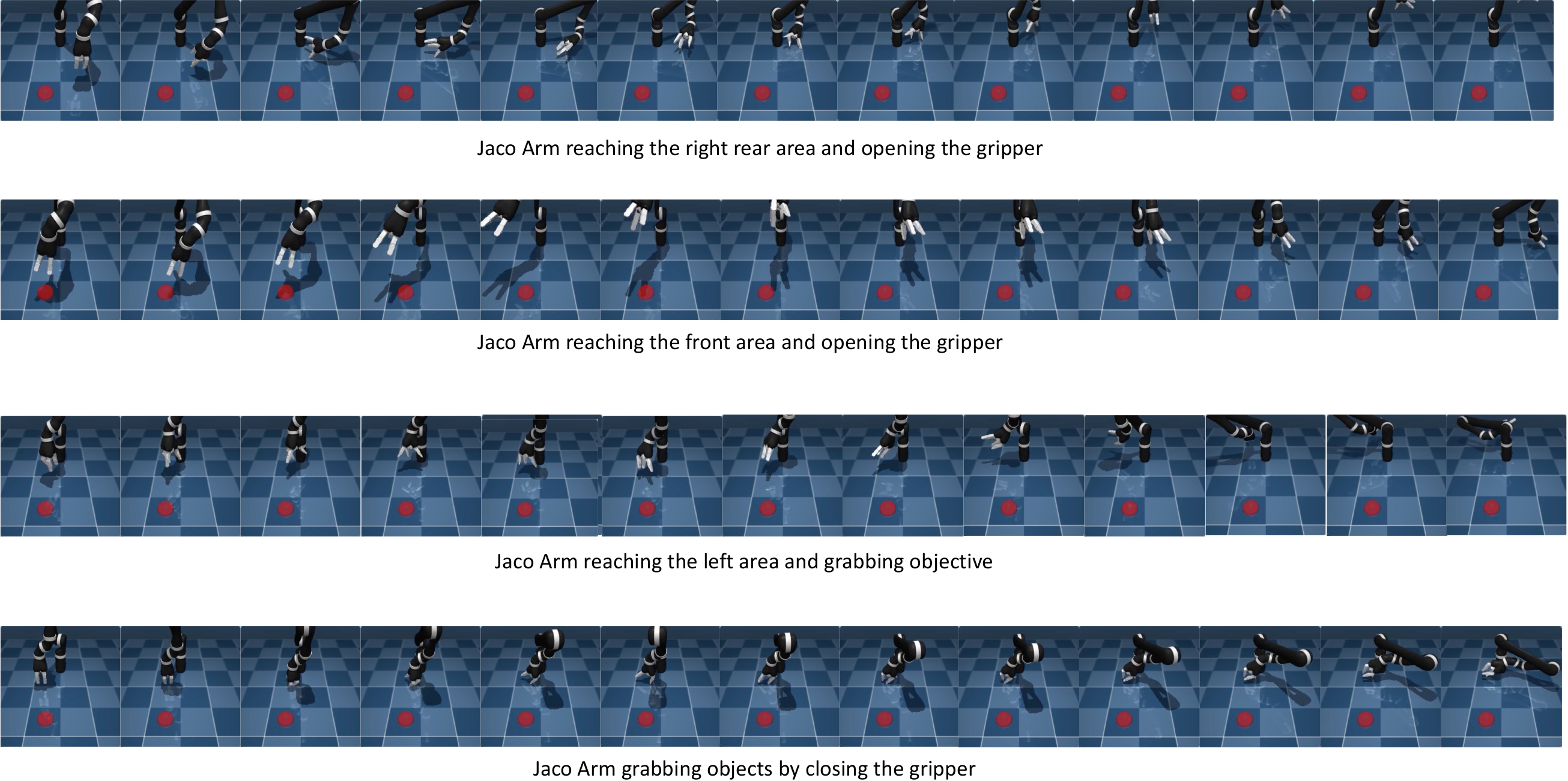}}
\caption{The Jaco Arm agent learns various manipulation skills, including reaching different locations, which allows fast adaptation in downstream tasks. The agent also learned to open and close grippers to manipulate objects. 
}
\label{fig:vis-skill-jaco}
\end{center}
\end{figure}

\newpage
\subsection{Numerical Results}
\label{app:dmc-score}
We report the individual normalized return of different methods in state-based URLB after 2M steps of pretraining and 100k steps of finetuning, as shown in Table~\ref{table:numerical_result}. In the \emph{Quadruped} and \emph{Jaco} domains, BeCL obtains state-of-the-art performance in downstream tasks. In the \emph{Walker} domain, CeSD shows competitive performance against the leading baselines.

\begin{table}[h!]
\centering
\caption{Results of CeSD and other baselines on state-based URLB. All baselines are pre-trained for 2M steps with only intrinsic rewards in each domain and then finetuned to 100K steps in each downstream task by giving the extrinsic rewards. All baselines are run for 10 seeds per task. The highest scores are highlighted.}
\vspace{0.5em}
\label{table:numerical_result}
\resizebox{\columnwidth}{!}{
\begin{tabular}{c|c|ccc|cccc|cccc|c}
\toprule
Domain                     & Task               & DDPG   & ICM        & Disagreement & RND        & APT        & ProtoRL    & SMM        & DIAYN      & APS        & CIC        & BeCL  & CeSD     \\ \hline
\multirow{4}{*}{Walker}    & {Flip}               & {538±27} & 390$\pm$10 & 332$\pm$7    & {506$\pm$29} & 606$\pm$30 & {549$\pm$21} & 500$\pm$28 & 361$\pm$10 & 448$\pm$36 & {\textbf{641$\pm$26}} & 611$\pm$18 & 541$\pm$17 \\
& {Run}                & {325±25} & 267$\pm$23 & 243$\pm$14   & {403$\pm$16} & 384$\pm$31 & {370$\pm$22} & 395$\pm$18 & 184$\pm$23 & 176$\pm$18 & {\textbf{450$\pm$19}} & 387$\pm$22 & 337$\pm$19\\
& {Stand}              & {899±23} & 836$\pm$34 & 760$\pm$24   & {901$\pm$19} & 921$\pm$15  & {896$\pm$20} & 886$\pm$18 & 789$\pm$48 & 702$\pm$67 & 959$\pm$2  & 952$\pm$2 & \textbf{960$\pm$3} \\
& {Walk}               & {748±47} & 696$\pm$46 & 606$\pm$51   & {783$\pm$35} & 784$\pm$52 & {836$\pm$25} & 792$\pm$42 & 450$\pm$37 & 547$\pm$38 & {\textbf{903$\pm$21}} & 883$\pm$34 & 834$\pm$34 \\ \hline
\multirow{4}{*}{Quadruped} & {Jump}               & {236±48} & 205$\pm$47 & 510$\pm$28   & {626$\pm$23} & 416$\pm$54 & {573$\pm$40} & 167$\pm$30 & 498$\pm$45 & 389$\pm$72 & {565$\pm$44} & 727$\pm$15 & \textbf{755$\pm$14} \\
& {Run}                & {157±31} & 125$\pm$32 & 357$\pm$24   & {439$\pm$7}  & 303$\pm$30 & {324$\pm$26} & 142$\pm$28 & 347$\pm$47 & 201$\pm$40 & {445$\pm$36} & 535$\pm$13 & \textbf{586$\pm$25}\\
& {Stand}              & {392±73} & 260$\pm$45 & 579$\pm$64   & 839$\pm$25 & 582$\pm$67 & {625$\pm$76} & 266$\pm$48 & 718$\pm$81 & 435$\pm$68 & {700$\pm$55} & 875$\pm$33 & \textbf{919$\pm$11} \\
& {Walk}               & {229±57} & 153$\pm$42 & 386$\pm$51   & {517$\pm$41} & 582$\pm$67 & {494$\pm$64} & 154$\pm$36 & 506$\pm$66 & 385$\pm$76 & {621$\pm$69} & 743$\pm$68 & \textbf{889$\pm$23} \\ \hline
\multirow{4}{*}{Jaco}      & {Re. bottom left}  & {72±22}  & 88$\pm$14  & 117$\pm$9    & {102$\pm$9}  & 143$\pm$12 & {118$\pm$7}  & 45$\pm$7   & 20$\pm$5   & 84$\pm$5   & {154$\pm$6}  & 148$\pm$13 & \textbf{208$\pm$5} \\
& {Re. bottom right} & {117±18} & 99$\pm$8   & 122$\pm$5    & {110$\pm$7} & 138$\pm$15 & 138$\pm$8  & 60$\pm$4   & 17$\pm$5   & 94$\pm$8   & 149$\pm$4 & 139$\pm$14 & \textbf{186$\pm$13} \\
& {Re. top left}     & {116±22} & 80$\pm$13  & 121$\pm$14   & {88$\pm$13}  & 137$\pm$20 & {134$\pm$7}  & 39$\pm$5   & 12$\pm$5   & 74$\pm$10  & 149$\pm$10 & 125$\pm$10 & \textbf{215$\pm$4}\\
& {Re. top right}    & {94±18}  & 106$\pm$14 & 128$\pm$11   & {99$\pm$5}   & 170$\pm$7 & {140$\pm$9}  & 32$\pm$4   & 21$\pm$3   & 83$\pm$11  & {163$\pm$9}  & 126$\pm$10 & \textbf{195$\pm$9} \\ \hline           
\end{tabular}}
\end{table}

\section{More Discussions}

\subsection{Difference to Mixture-of-Expert (MoE) \cite{celik2022specializing}} 

The fundamental difference is the problem setting. We focus on unsupervised skill discovery, aiming to learn distinguishable skills without extrinsic reward and task structure information, for efficiently solving downstream tasks via finetuning skills. In contrast, the MoE work addresses learning skills in the context-conditioned tasks with extrinsic reward, where different tasks are represented by different contexts $c$. Correspondingly, the learned MoE model depends on the context for skill inference (i.e., $\pi(\theta|c) =\sum_{o\in O} \pi(o|c) \pi(\theta|o,c)$, where $o$ represents skills/components). Furthermore, the downstream task provides explicit context for the algorithm, which makes the method less general. Thus, the MoE method may not be deployed directly to the unsupervised skill discovery as far as we know; we can only receive the states from the environment and encourage diverse behaviors via some self-proposed objective (such as $r = \log(z|s)$ from DIAYN).

Second, the details of the method are quite different. As for maximizing state coverage, we propose portioned exploration to encourage local skill exploration, while the MoE algorithm uses policy entropy (i.e., $H(\pi(\theta|o,c))$), which is common in RL research. As for distinguishing between skills, we propose the clustering-based technique. In contrast, the MoE algorithm does not introduce the technique to explicitly encourage skill/component diversity as we know. Given the context-conditioned task (e.g., the context $c$ defines the target position in table tennis) and the context-conditioned extrinsic reward function $r(s,a,c)$, the components/skills can naturally derive distinguishable behaviors under the guidance of the context. Imagine a simple case, we train multiple skill networks, and each skill is trained to maximize its own context-based reward function (i.e., $\pi_i^* = \arg\max \mathbb{E} _{\pi_i}[\sum _{t=0}^\infty r(s,a,c_i)]$), the trained skill will obtain distinguishable behaviors finally (e.g., different skill plays table tennis towards different target positions).

\begin{table}[t]
\centering
\caption{Result comparison of MoE methods.}
\vspace{0.5em}
\label{table:numerical_result}
\small
\begin{tabular}{lccc}
\toprule
{\textbf{Task}}    & {\textbf{CeSD+MoE (Finetune skill)}} & { \textbf{CeSD+MoE (Freeze skill)}} & { \textbf{CeSD}} \\\hline
{ walker\_stand}    & { 341 ± 16}                          & { 339 ± 48}                        & { \textbf{960} ± 3}       \\
{ walker\_run}      & { 75 ± 4}                            & { 71 ± 8}                          & { \textbf{337} ± 19}      \\
{ walker\_walk}     & { 157 ± 9}                           & { 159 ± 13}                        & { \textbf{834} ± 34}      \\
{ walker\_flip}     & { 197 ± 8}                           & {200 ± 13}                        & { \textbf{541} ± 17}      \\
{ quadruped\_stand} & { 627 ± 203}                         & { 532 ± 101}                       & { \textbf{919} ± 11}      \\
{ quadruped\_jump}  & { 480 ± 147}                         & { 361 ± 151}                       & { \textbf{755} ± 14}      \\
{ quadruped\_run}   & { 327 ± 107}                         & { 297 ± 60}                        & { \textbf{586} ± 25}      \\
{ quadruped\_walk}  & { 295 ± 158}                         & { 255 ± 70}                        & { \textbf{889} ± 23}    
\\\bottomrule
\end{tabular}
\end{table}

\subsection{Calculation of Eq.~\eqref{eq:reg_r}}

The size of $|\mathbb{S}^{\rm pe}_i-\mathbb{S}^{\rm clu}_i|$ is easy to calculate since the two state-sets are mostly overlapped. In clustering, for state $s\in \mathbb{S}^{\rm pe}_i$ collected by the skill policy $\pi_i$, (\romannumeral1) if $s$ is collected in the previous rounds, we have the cluster label unchanged (i.e., $(s, a, s')\in\mathbb{S}^{\rm clu}_i$) since the Sinkhorn-Knopp cluster algorithm will keep the cluster-index of existing states fixed; and (\romannumeral2) if $(s, a, s')$ is the newly collected one in the current round, it may be assigned to cluster $i$ or other clusters (e.g., $j$) according to $\{f(s)^\top c_j\} _{j\in[n]}$. Then we use $r=1/(|\mathbb{S}^{\rm pe}_i-\mathbb{S}^{\rm clu} _i|+\lambda)$ as the rewards to force $\pi_i$ to reduce the visitation probability of states lied in clusters of other skills. In implementation, we give each transition $(s,a,s')$ two skill labels (i.e., $z^{\rm pe}$ and $z^{\rm clu}$). Specifically, $z^{\rm pe}$ signifies the transition $(s,a,s')$ is collected by which skill policy in exploration, and $z^{\rm clu}$ is determined by the clustering index of Sinkhorn-Knopp algorithm.

\subsection{Non-overlapping Property}

The non-overlapping property of skills is not a hard constraint in our method but a soft one with a tolerance value. In Sec. 3.3, we define a desired policy $\hat{\pi}_i$ based on the skill policy $\pi_i$, where $d^{\hat{\pi}_i}(s)=0$ for overlapping states between clusters. Then our constraint for regularizing skill $\pi_i$ is defined as $\mathcal{L} _{\rm reg}(\pi_i)=D _{\rm TV}(d^{\hat{\pi}_i} \| d^{\pi_i})$. In practice, we adopt a heuristic intrinsic reward to prevent the policy $\pi_i$ from visiting states in $\mathbb{S}^{\rm pe}_i-\mathbb{S}^{\rm clu} _i$, as $r^{\rm reg} _i=1/(|\mathbb{S}^{\rm pe} _i-\mathbb{S}^{\rm clu}_i|+\lambda)$, which we assume to make the TV-distance between state distributions bounded by $D _{\rm TV}\big(d^{\hat{\pi}_i}|d^{\pi_i}\big)\leq \delta$, where $\delta$ is a tolerance value. In our paper, Theorem 3.3 and Corollary 3.4 hold with such a tolerance value. Some special cases exist in which each skill policy must visit some bottleneck states. In these cases, the regularization reward $r^{\rm reg}_i=1/(|\mathbb{S}^{\rm pe}_i-\mathbb{S}^{\rm clu}_i|+\lambda)\leq 1/(c+\lambda)$, where $c$ is the number of bottleneck states. The reward $r^{\rm reg}_i$ will become small if $c$ is very large, which can be alleviated by removing this constant or increasing the weight of $r^{\rm reg}_i$ in policy updating. 

\subsection{Pixel-based URLB}

According to Pixel-URLB \cite{rajeswar2023mastering}, which evaluates the unsupervised RL algorithms on pixel-based URLB, the performance in pixel-based URLB depends heavily on the basic RL algorithm. Specifically, according to Figure 1 of \citet{rajeswar2023mastering}, all unsupervised RL algorithms perform poorly when combined with a model-free method (e.g., DrQv2), while they perform much better when using a model-based algorithm (e.g., Dreamer) as the backbone. APT obtains the best performance in the challenging Quadruped domain compared to other methods. Following the official code of [1], we re-implement CeSD with the Dreamer backbone. We compare CeSD-Dreamer and APT-Dreamer in the following table. The result shows our method outperforms APT in the pixel-based domain on average.

\begin{table}[t]
\centering
\caption{Results comparison of Pixel-based URLB methods.}
\vspace{0.5em}
\label{table:pixel-urlb}
\small
\begin{tabular}{lcc}
\toprule
{ \textbf{Pixel-based Task}} & { \textbf{CeSD Dreamer}} & { \textbf{APT Dreamer}} \\\hline
{ quadruped\_jump}           & { \textbf{756.5} ± 60}            & { 584.6 ± 1}            \\
{ quadruped\_run}            & { \textbf{445.7} ± 23}            & { 428.2 ± 21}           \\
{ quadruped\_stand}          & { 864.9 ± 1}             & { \textbf{914.7} ± 7}            \\
{ quadruped\_walk}           & { \textbf{581.5} ± 129}           & { 473.7 ± 27}          
\\\bottomrule
\end{tabular}
\end{table}

\subsection{Discrete/Continuous Skill Space}

Although infinite skills (in continuous space) seem to be a better choice, infinite skills do not always lead to better performance than discrete ones. As shown in Fig. 3, DADS have a continuous skill space while the resulting state coverage is limited. As for the DMC tasks, the baseline methods, including APS, DADS, and CIC, also have a continuous skill space. Actually, learning infinite skills with diverse and meaningful behaviors is desirable, while it can be difficult for existing skill discovery methods. In our method, since we adopt partition exploration based on Sinkhorn-Knopp clustering, the cluster number is required to be finite to partition the state space, and each state should be assigned to a specific cluster.

\subsection{Additional Comparison to Re-Implemented Baselines}
\label{app:add-compare}

The recently proposed Metra \cite{metra} uses Wasserstein dependency to measure (WDM) between states and skills, i.e., $I_W(S,Z)$, for skill discovery. Metra also contains experiments in URLB benchmark while it only reports the skill policy’s coverage (see Fig. 5 of \citet{metra}), and the downstream tasks are specifically designed to reach a target goal (see Appendix F.1 of \citet{metra}) rather than diverse task adaptation considered in our paper. As a result, we use the official Metra code and carefully modify the goal adaptation process to evaluate the adaptation of various DMC tasks. We also add new baselines, including LSD \cite{LSD-2022} and CSD \cite{USD-2023}. Since LSD/CSD are evaluated on different benchmarks in their original papers, we have tried our best to re-implement LSD/CSD in URLB tasks based on the official code. A comparison of the results is given in the following table. We find out that the method obtains competitive performance compared to CSD and Metra in the \emph{Walker} domain and significantly outperforms other methods in the \emph{Quadruped} domain.

\begin{table}[h!]
\centering
\caption{Results comparison to re-implemented baselines.}
\vspace{0.1em}
\small
\label{table:add-urlb}
\begin{tabular}{lcccc}
\\\toprule
\textbf{Task}                           & \textbf{LSD}                     & \textbf{CSD}                     & \textbf{Metra}                   & \textbf{CeSD}                   \\\hline
{ walker\_flip}     & { 223 ± 6}   & { \textbf{602} ± 11}  & { 589 ± 75}  & { 541 ± 17} \\
{ walker\_run}      & { 130 ± 22}  & { \textbf{457} ± 50}  & { 361 ± 45}  & { 337 ± 19} \\
{ walker\_stand}    & { 837 ± 3}   & { 942 ± 8}   & { 943 ± 13}  & { \textbf{960} ± 3}  \\
{ walker\_walk}     & { 323 ± 75}  & { 802 ± 85}  & { \textbf{850} ± 63}  & { \textbf{834} ± 34} \\
{ quadruped\_jump}  & { 247 ± 54}  & { 520 ± 80}  & { 224 ± 17}  & { \textbf{775} ± 14} \\
{ quadruped\_run}   & { 270 ± 55}  & { 329 ± 62}  & { 196 ± 34}  & { \textbf{586} ± 25} \\
{ quadruped\_stand} & { 426 ± 131} & { 425 ± 120} & { 324 ± 173} & { \textbf{919} ± 11} \\
{ quadruped\_walk}  & { 256 ± 83}  & { 353 ± 142} & { 190 ± 44}  & { \textbf{889} ± 23}
\\\bottomrule
\end{tabular}
\end{table}

\end{document}